\documentclass{article}

\usepackage[main, final]{neurips_2026}
 \usepackage{amsmath}
 \usepackage{mathtools}
 \usepackage{amssymb}
 \usepackage{amsthm}
\usepackage{wrapfig}
\usepackage{algorithm}
 \usepackage{booktabs}
\usepackage{algpseudocode}

\usepackage[utf8]{inputenc} 
\usepackage[T1]{fontenc}    
\usepackage{hyperref}       
\usepackage{url}            
\usepackage{amsfonts}       
\usepackage{nicefrac}       
\usepackage{microtype}      
\usepackage{xcolor}         
\newtheorem{theorem}{Theorem}
\newtheorem{proposition}{Proposition}
\newtheorem{lemma}{Lemma}
\newtheorem{corollary}[theorem]{Corollary}

\usepackage{multirow}
\usepackage{graphicx}
\newcommand{\score}[2]{#1 {\footnotesize #2}}
\title{TGRL: Temperature-Grouped Reinforcement Learning for Efficient Exploration in LLMs}

\author{Zihan Lin$^{1,2,3}$\thanks{Both authors contributed equally to this research.}\hspace{0.4em}\thanks{Work was done during an internship at Meituan.}\hspace{0.4em}, Xiaohan Wang$^{2}$\footnotemark[1] \thanks{Corresponding authors: \texttt{wangxiaohan17@meituan.com}, \texttt{yinguojun02@meituan.com},\texttt{ran.he@ia.ac.cn}} , Jie Cao$^{3}$, Jiajun Chai$^{2}$, \\
 \textbf{Wei Lin$^{2}$, Guojun Yin$^{2}$\footnotemark[3] , Ran He$^{1,3}$\footnotemark[3]} \\
\\
$^1$University of Chinese Academy of Sciences \\
$^2$Meituan \\
$^3$MAIS\&NLPR, Institute of Automation, Chinese Academy of Sciences
}

\begin{document}

\maketitle









\begin{abstract}
Efficient exploration often remains a central bottleneck in reinforcement learning with verifiable rewards (RLVR). Although temperature control and test-time scaling strategies can increase rollout diversity of large language models (LLMs), they either expand the sample budget at rollout time or leave the benefit of exploration unquantified.
To this end, we propose Temperature-Grouped Reinforcement Learning (TGRL), which turns temperature-induced diversity into an explicit training signal. For each prompt, TGRL partitions its rollout group into low- and high-temperature subsets, estimates exploration gain through their reward contrast, and allocates this group-level signal as token-level credit using Jensen--Shannon (JS) divergence between the corresponding temperature-scaled next-token distributions induced by the same logits. Notably, TGRL reaches equivalent accuracy up to $36\%$ faster than strong RLVR baselines without expanding the rollout budget.  Across 11 benchmarks from diverse domains, TGRL broadly improves over strong RLVR baselines: it improves the six-benchmark math average by $1.6\%$ at 32B, raises CodeForces rating by $196.7$ points and LiveCodeBench Pass@16 by $4.4\%$, and improves ALFWorld/WebShop success rates by $6.3\%$/$4.9\%$. Comprehensive ablations and wall-clock analysis confirm the efficacy of all proposed components. Code is available at \url{https://github.com/1229095296/TGRL/tree/main}.
\end{abstract}

\section{Introduction}

Recently, reinforcement learning with verifiable rewards (RLVR) has become a dominant post-training paradigm for improving the reasoning capabilities of large language models (LLMs), enabling iterative self-improvement through policy-gradient optimization over automatically verifiable outcomes \citep{guo2025deepseek,lin2026resrlboostingllmreasoning,wang2025ragen,lin2026awpoenhancingtooluselarge}. However, the efficacy of RLVR depends critically on efficient exploration—the capacity to identify diverse, high-reward trajectories within a constrained rollout budget. Exploration dynamics not only dictate the policy's ability to escape local optima but also determine the attainable performance ceiling \citep{cui2025entropy,chen2025pass,wang2026self,hu2026ziprl}. Without robust exploration and sufficient trajectory diversity, RLVR optimization can prematurely converge to a narrow decoding mode, thereby limiting the model’s attainable reasoning gains.

To enable efficient exploration in RLVR, existing work usually focuses on increasing generation diversity during rollout. Common approaches include test-time scaling, which draws more samples \citep{wang2022self, brown2024large}, and decoding interventions such as temperature control or temperature schedules \citep{yang2025let,dang2026temperature}. However, these strategies either expand the sample budget at rollout time or leave the estimated benefit of exploration unquantified. We argue that efficient RLVR exploration requires both an active sampling intervention and an explicit estimate of exploration gain, so that the exploration induced at the rollout stage can be connected to the policy-gradient update optimized during training.

To instantiate this principle, we propose Temperature-Grouped Reinforcement Learning (TGRL), a novel RLVR framework that estimates exploration gain and refines it into token-level credit. As illustrated in Figure~\ref{fig:tgrl_overview}, TGRL samples low-temperature reference rollouts and high-temperature exploration rollouts for each prompt. The mean-reward gap between the two groups estimates the prompt-level exploration gain, which is injected into the mixed-group advantage to produce a reward signal that reflects both within-group ranking and exploration benefit. TGRL then computes token-level JS divergence from the same logits and allocates credit proportionally, concentrating gradient updates on temperature-sensitive positions. In this way, TGRL translates the reward benefit of temperature-induced exploration into token-level policy-gradient credit. Built on this design, our contributions are as follows:
\begin{enumerate}
\item \textbf{Theoretical characterization of exploration gain and token-level credit allocation.}
We formalize temperature as a controllable perturbation of the same policy, showing that the reward contrast between low- and high-temperature groups estimates prompt-level exploration gain. We also characterize token-level JS divergence as a local distributional response to the temperature perturbation, providing a principled basis for allocating more credit to temperature-sensitive positions.

    \item \textbf{Exploration-gain estimation with JS-based token credit allocation.}
    We introduce TGRL, which combines two mechanisms: mixed-temperature grouping estimates whether broader exploration improves the verifiable reward for each prompt, and JS-based credit allocation maps this grouped signal into token-level policy-gradient credit for high-temperature exploratory trajectories.

    \item \textbf{Strong empirical gains across reasoning, coding, and agent tasks.}
    Across 11 benchmarks, TGRL broadly improves over strong RLVR baselines: it improves the six-benchmark math average by $1.6\%$ at 32B, raises CodeForces rating by $196.7$ points and LiveCodeBench Pass@16 by $4.4\%$, and improves ALFWorld/WebShop success rates by $6.3\%$/$4.9\%$. Ablations, training-dynamics analyses, and wall-clock measurements further support the proposed grouping-and-credit-allocation mechanism.
\end{enumerate}

\begin{figure}[t]
    \centering
    \includegraphics[width=\textwidth]{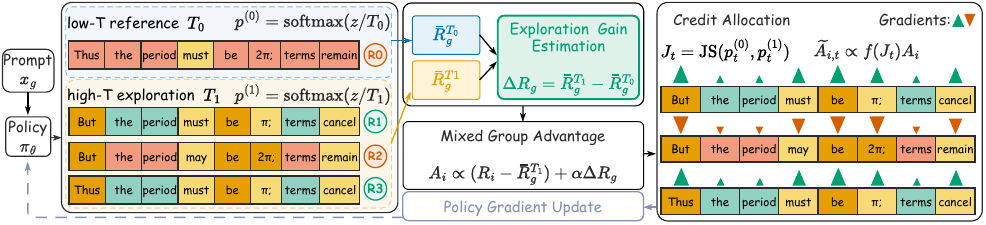}
\caption{\textbf{Overview of TGRL.} TGRL forms low-temperature reference and high-temperature exploration groups to estimate exploration gain from their mean-reward gap, and then assigns this gain to temperature-sensitive tokens via token-level JS divergence for policy-gradient updates.}
    \label{fig:tgrl_overview}
    \vspace{-0.5em}
\end{figure}

\section{Related Work}
\label{related_work}

\paragraph{RLVR optimization algorithms.}
RLVR has become a core post-training paradigm for improving mathematical, coding, and agentic reasoning in LLMs \citep{shao2024deepseekmath,lu2026contextual,yang2026your}. 
Most RLVR algorithms instantiate this paradigm through policy-gradient optimization over sampled completions: PPO provides the canonical clipped template \citep{ouyang2022training}, RLOO uses leave-one-out baselines \citep{ahmadian2024back}, and REINFORCE++ incorporates PPO-style stabilization into critic-free REINFORCE training to reduce overhead \citep{hu2025reinforcepp}. 
Building on this foundation, GRPO establishes group-relative normalization over model-generated responses as a standard RLVR recipe for verifiable reasoning \citep{shao2024deepseekmath,liu2026cdrrm}. 
Subsequent methods primarily refine this group-based framework by correcting objective biases or improving large-scale training stability: Dr.~GRPO and VAPO address biases from length effects, sparse rewards, and value estimation \citep{liu2025drgrpo,yue2025vapo}, while DAPO and GSPO improve stability through decoupled clipping, dynamic sampling, and sequence-level importance control \citep{yu2025dapo,zheng2025gspo}. 
Collectively, these methods substantially strengthen the optimization side of RLVR. However, they commonly assume that the rollout groups already contain sufficiently diverse trajectories, thereby neglecting the active estimation and exploitation of exploration gain within a fixed budget.

\paragraph{Exploration strategies in RL.}
Despite optimizer advances, exploration during rollout generation remains a bottleneck in RLVR. Recent analyses link this issue to shrinking trajectory diversity, insufficient long-horizon search, and the mismatch between single-trajectory optimization and multi-sample success criteria \citep{cui2025entropy,cheng2026reasoning,chen2025pass}. Classical RL studies exploration through uncertainty-driven, novelty- or curiosity-based, perturbation-based, and maximum-entropy mechanisms \citep{jaksch2010nearoptimal,osband2016bootstrapped,tang2017exploration,pathak2017curiosity,plappert2017parameter,fortunato2018noisy,haarnoja2018soft}. In LLM inference, test-time scaling promotes exploration by sampling, searching over, or verifying multiple reasoning trajectories before selecting or aggregating outputs \citep{wang2022self,yao2023tree,brown2024large,snell2024scaling}. While effective, these methods typically rely on increasing the number of sampled trajectories or rollouts, making exploration costly. Training-time RLVR methods instead incorporate exploration into optimization through rollout filtering and grouping \citep{xu2025not,chen2025dra}, entropy/search/pass@$k$ objectives \citep{cui2025entropy,cheng2026reasoning,chen2025pass,walder2025pass}, and temperature scheduling or policies \citep{yang2025let,dang2026temperature,zhou2026lookinward}. Nevertheless, they mainly treat exploration as a global regularizer, auxiliary objective, or decoding policy, rather than as a reward-grounded, prompt-level exploration-gain signal induced by mixed-temperature grouping. TGRL fills this gap by estimating exploration gain from reward gaps between low- and high-temperature groups and allocating the signal as token-level credit to temperature-sensitive tokens via token-level JS divergence.

\section{Method}
\label{sec:method}

We first establish the theoretical basis for the two mechanisms of TGRL (Section~\ref{sec:theory_analysis}), then describe their joint instantiation in the TGRL algorithm (Section~\ref{sec:algorithm}).

\subsection{Theoretical Analysis}
\label{sec:theory_analysis}

We analyze the theoretical roles of TGRL's two mechanisms: how mixed-temperature grouping produces an exploration signal, and how this signal can be refined into token-level credit. First, we show that the low-/high-temperature reward gap estimates prompt-level exploration gain and appears as an additive shift in mixed-group advantages. Second, we show that token-level JS divergence captures local sensitivity to the temperature intervention, thereby allocating the signal to temperature-sensitive tokens.

\paragraph{Setup.}
For each prompt $x_g$, TGRL draws low- and high-temperature rollouts indexed by $\mathcal{G}_{g,T_0}$ and $\mathcal{G}_{g,T_1}$, respectively, forming $\mathcal{G}_{g,T_0+T_1}\coloneqq\mathcal{G}_{g,T_0}\cup\mathcal{G}_{g,T_1}$ with $|\mathcal{G}_{g,T_0+T_1}|=|\mathcal{G}_{g,T_0}|+|\mathcal{G}_{g,T_1}|$. Rollout $i\in\mathcal{G}_{g,T_0+T_1}$ yields a scalar reward $R_i$ and logits $z_{i,t}\in\mathbb{R}^{|V|}$ at response position $t$. We define
\begin{equation}
\label{eq:js_setup}
p_{i,t}^{(0)} \coloneqq \mathrm{softmax}(z_{i,t}/T_0),\qquad
p_{i,t}^{(1)} \coloneqq \mathrm{softmax}(z_{i,t}/T_1),\qquad
J_{i,t} \coloneqq \mathrm{JS}\!\bigl(p_{i,t}^{(0)},p_{i,t}^{(1)}\bigr),
\end{equation}
where $\mathrm{JS}$ denotes the Jensen--Shannon divergence \citep{lin2002divergence}, a symmetric and bounded measure of discrepancy between probability distributions, as explained in Appendix~\ref{app:js_allocation_appendix}. Thus $J_{i,t}$ measures how strongly the local next-token distribution responds to the temperature switch from $T_0$ to $T_1$.

\subsubsection{Exploration Gain Estimation}
\label{sec:theory_exploration_gain}

\begin{lemma}[Prompt-level exploration gain under mixed-temperature grouping]
\label{lem:dual_var}
For a fixed prompt $x$, suppose the low-temperature rewards are independent draws from $\pi_\theta^{T_0}(\cdot\mid x)$ and the high-temperature rewards are independent draws from $\pi_\theta^{T_1}(\cdot\mid x)$, independently across the two subgroups. Define
\begin{equation}
\label{eq:dual_gap_def}
    \bar{R}_g^{T_0} \coloneqq \frac{1}{|\mathcal{G}_{g,T_0}|}\sum_{i\in\mathcal{G}_{g,T_0}} R_i,
    \qquad
    \bar{R}_g^{T_1} \coloneqq \frac{1}{|\mathcal{G}_{g,T_1}|}\sum_{i\in\mathcal{G}_{g,T_1}} R_i,
    \qquad
    \Delta R_g \coloneqq \bar{R}_g^{T_1} - \bar{R}_g^{T_0}.
\end{equation}
Then
\begin{equation}
\label{eq:dual_var}
    \mathbb{E}[\Delta R_g \mid x] = \mu_{T_1}(x) - \mu_{T_0}(x),
    \qquad
    \mathrm{Var}(\Delta R_g \mid x) = \frac{\sigma_{T_1}^2(x)}{|\mathcal{G}_{g,T_1}|} + \frac{\sigma_{T_0}^2(x)}{|\mathcal{G}_{g,T_0}|},
\end{equation}
where $\mu_T(x) \coloneqq \mathbb{E}_{y\sim\pi_\theta^T(\cdot\mid x)}[R(x,y)]$ and $\sigma_T^2(x) \coloneqq \mathrm{Var}_{y\sim\pi_\theta^T(\cdot\mid x)}[R(x,y)]$.
\end{lemma}
\noindent Hence $\Delta R_g$ is an unbiased estimate of the prompt-level exploration gain. Uniform-temperature group normalization has no within-prompt temperature contrast and therefore cannot recover this quantity directly (Appendix~\ref{app:grpo_gap}).

\begin{proposition}[Mixed-group advantage decomposition]
\label{prop:mixed_adv}
Let $\mathcal{G}_{g,T_0+T_1} \coloneqq \mathcal{G}_{g,T_0} \cup \mathcal{G}_{g,T_1}$, and define
\begin{equation}
\label{eq:mixed_group_mean}
    \mu_g \coloneqq \frac{1}{|\mathcal{G}_{g,T_0+T_1}|}\sum_{i\in \mathcal{G}_{g,T_0+T_1}} R_i,
    \qquad
    \sigma_g^2 \coloneqq \frac{1}{|\mathcal{G}_{g,T_0+T_1}|}\sum_{i\in \mathcal{G}_{g,T_0+T_1}}(R_i-\mu_g)^2.
\end{equation}
Consider the mixed-group advantage
\begin{equation}
\label{eq:mixed_group_adv}
    A_i \coloneqq \frac{R_i-\mu_g}{\sqrt{\sigma_g^2+\epsilon}},
\end{equation}
where $\epsilon>0$ is a numerical stabilizer. Then
\begin{equation}
\label{eq:mixed_group_decomp}
    A_i
    =
    \begin{cases}
        \dfrac{R_i-\bar{R}_g^{T_0}}{\sqrt{\sigma_g^2+\epsilon}} - \dfrac{|\mathcal{G}_{g,T_1}|}{|\mathcal{G}_{g,T_0+T_1}|}\,\dfrac{\Delta R_g}{\sqrt{\sigma_g^2+\epsilon}}, & i \in \mathcal{G}_{g,T_0}, \\[10pt]
        \dfrac{R_i-\bar{R}_g^{T_1}}{\sqrt{\sigma_g^2+\epsilon}} + \dfrac{|\mathcal{G}_{g,T_0}|}{|\mathcal{G}_{g,T_0+T_1}|}\,\dfrac{\Delta R_g}{\sqrt{\sigma_g^2+\epsilon}}, & i \in \mathcal{G}_{g,T_1}.
    \end{cases}
\end{equation}
\end{proposition}
\noindent Proposition~\ref{prop:mixed_adv} shows that mixed-group normalization shifts every high-temperature advantage by an additive term proportional to the exploration gain $\Delta R_g$. Averaging Eq.~\eqref{eq:mixed_group_decomp} over each subgroup yields the corresponding subgroup-mean identity (Appendix~\ref{app:subgroup_means}). A synthetic diagnostic in which prompt-level exploration gain is provided in Appendix~\ref{app:grouping_validation_details}.

\subsubsection{Token-Level Credit Allocation}
\label{sec:theory_js_allocation}

Mixed-temperature grouping determines \emph{whether} broader exploration yields a positive reward gain. We now characterize \emph{how} JS-based weights distribute that signal across token positions.

\begin{lemma}[Global temperature-response control of token-level JS]
\label{lem:js_local}
Let $\beta_0 \coloneqq 1/T_0$, $\beta_1 \coloneqq 1/T_1$, and $\Delta\beta \coloneqq \beta_1-\beta_0$. For fixed logits $z_{i,t}$, define $p_{i,t}^{(\beta)} \coloneqq \mathrm{softmax}(\beta z_{i,t})$. Then
\begin{equation}
\label{eq:js_expansion}
    J_{i,t}
    \le
    \frac{(\Delta\beta)^2}{8}
    \int_{0}^{1}
    \mathrm{Var}_{a\sim p_{i,t}^{(\beta_0+s\Delta\beta)}}[z_{i,t,a}]\,ds.
\end{equation}
\end{lemma}
\noindent Lemma~\ref{lem:js_local} bounds token-level JS by the path-averaged logit variance along the interpolation from $T_0$ to $T_1$. Larger $J_{i,t}$ therefore identifies positions whose next-token distribution is more sensitive to the temperature intervention.

To state the allocation result, let $v_t(a)$ denote the normalized continuation value of choosing token $a$ at position $t$; concretely, it can be interpreted as the expected downstream reward after fixing token $a$ at that position and continuing the rollout. We assume $v_t(a)\in[0,1]$.

\begin{proposition}[Monotone JS weighting allocates more credit to temperature-sensitive positions]
\label{prop:credit_quality}
Fix a high-temperature rollout $i$ with valid positions $\mathcal{T}_i$. For each $t\in\mathcal{T}_i$, let $v_t\colon \{1,\dots,|V|\}\to[0,1]$ be such a bounded continuation-value map. Then
\begin{equation}
\label{eq:local_effect_bound}
\left|
\mathbb{E}_{a\sim p_{i,t}^{(1)}}[v_t(a)]
-
\mathbb{E}_{a\sim p_{i,t}^{(0)}}[v_t(a)]
\right|
\le
\sqrt{2J_{i,t}},
\qquad t\in\mathcal{T}_i.
\end{equation}
Moreover, if token weights are defined by
\begin{equation}
\label{eq:theory_weight_profile}
q_t \coloneqq \frac{h_i(J_{i,t})}{\sum_{s\in\mathcal{T}_i} h_i(J_{i,s})},
\qquad t\in\mathcal{T}_i,
\end{equation}
for any trajectory-wise non-decreasing $h_i:\mathbb{R}_+\to\mathbb{R}_+$ with a positive denominator, then
\begin{equation}
\label{eq:theory_weight_allocation}
\sum_{t\in\mathcal{T}_i} q_t\,\sqrt{2J_{i,t}}
\;\ge\;
\frac{1}{|\mathcal{T}_i|}\sum_{t\in\mathcal{T}_i} \sqrt{2J_{i,t}}.
\end{equation}
\end{proposition}
\noindent Proposition~\ref{prop:credit_quality} has two implications. First, $\sqrt{2J_{i,t}}$ upper-bounds how much the expected continuation value at position $t$ can change when the next-token distribution is switched from $T_0$ to $T_1$. Second, any monotone JS-based weighting allocates more mass than uniform weighting to positions with larger certified temperature-sensitivity budget. Appendix~\ref{app:proof_credit} gives the proof, the covariance and ratio forms of Eq.~\eqref{eq:theory_weight_allocation}, and the connection to the implemented weight function. Appendix~\ref{app:js_allocation_appendix} reports a local-cooling probe comparing JS-, entropy-, and margin-based span selectors: JS-targeted spans produce the highest flip rate ($12.1\%$) and score drop ($0.241$) at matched position ratios, validating temperature sensitivity as an effective criterion for token-level credit allocation in practice.

\subsection{Algorithm}
\label{sec:algorithm}

The algorithm instantiates the two theoretical roles established above: mixed-temperature grouping determines \emph{whether} broader exploration yields a positive reward gain, and token-level JS determines \emph{how} that exploration signal is allocated across token positions. Algorithm~\ref{alg:tgrl} gives the full procedure.

\paragraph{Mixed-temperature grouped advantage.}
Using the group partition and notation from Section~\ref{sec:theory_analysis}, TGRL computes the group-normalized advantage
$A_i = (R_i - \mu_g) / \sqrt{\sigma_g^2 + \epsilon}$,
as defined in Eq.~\eqref{eq:mixed_group_adv}.
The normalization is applied to a temperature-contrastive rollout group, consisting of low-temperature and high-temperature rollouts from the same prompt.
For a high-temperature rollout, the resulting advantage compares its verifiable reward against the mixed-temperature group, thereby estimating whether broader search improves the outcome beyond the low-temperature reference.

\paragraph{JS-based token credit allocation.}
The scalar $A_i$ is a rollout-level signal. For each high-temperature rollout $i\in\mathcal{G}_{g,T_1}$, we replay the sampled response to obtain logits $z_{i,t}$ at valid response positions $t\in\mathcal{T}_i$, where $\mathcal{T}_i$ is the response mask and $N_i\coloneqq |\mathcal{T}_i|$. We then compute
\begin{equation}
\label{eq:tgrl_js_def}
p_{i,t}^{(0)} \coloneqq \mathrm{softmax}(z_{i,t}/T_0),\qquad
p_{i,t}^{(1)} \coloneqq \mathrm{softmax}(z_{i,t}/T_1),\qquad
J_{i,t} \coloneqq \mathrm{JS}\!\bigl(p_{i,t}^{(0)},\,p_{i,t}^{(1)}\bigr),
\qquad t\in\mathcal{T}_i,
\end{equation}
as in Eq.~\eqref{eq:js_setup}. Let $m_{i,t} \coloneqq \tfrac{1}{2}\bigl(p_{i,t}^{(0)} + p_{i,t}^{(1)}\bigr)$ denote the midpoint distribution. Then
\begin{equation}
\label{eq:tgrl_js_expanded}
J_{i,t}
=
\frac{1}{2}\,\mathrm{KL}\!\bigl(p_{i,t}^{(0)}\,\|\,m_{i,t}\bigr)
+
\frac{1}{2}\,\mathrm{KL}\!\bigl(p_{i,t}^{(1)}\,\|\,m_{i,t}\bigr)
=
\frac{1}{2}\sum_{a=1}^{|V|} p_{i,t,a}^{(0)} \log \frac{p_{i,t,a}^{(0)}}{m_{i,t,a}}
+
\frac{1}{2}\sum_{a=1}^{|V|} p_{i,t,a}^{(1)} \log \frac{p_{i,t,a}^{(1)}}{m_{i,t,a}},
\end{equation}
where $a$ ranges over the vocabulary. By Proposition~\ref{prop:credit_quality}, $\sqrt{2J_{i,t}}$ upper-bounds the change in expected continuation value induced by the temperature switch, so larger $J_{i,t}$ marks positions with larger certified temperature-sensitivity budget. We map $J_{i,t}$ to practical token weights using a monotone log compression to smooth extreme JS-sensitivity spikes \citep{sakhi2024logarithmic}, followed by trajectory-wise normalization to keep the average token weight equal to one, 
\begin{equation}
\label{eq:tgrl_js_mean}
\bar J_i \coloneqq \frac{1}{N_i}\sum_{\tau\in\mathcal{T}_i} J_{i,\tau}, 
\qquad
\omega_{i,t}
\coloneqq
\log\!\left(1+\frac{J_{i,t}+\epsilon}{\bar J_i+\epsilon}\right),
\qquad
\bar \omega_i \coloneqq \frac{1}{N_i}\sum_{\tau\in\mathcal{T}_i} \omega_{i,\tau},
\qquad
w_{i,t}
\coloneqq
\frac{\omega_{i,t}}{\bar \omega_i}.
\end{equation}

We then define the token-level TGRL advantage
\begin{equation}
\label{eq:tgrl_token_adv}
\widetilde A_{i,t}
=
\begin{cases}
0, & i\in\mathcal{G}_{g,T_0}, \\[4pt]
w_{i,t}\,A_i, & i\in\mathcal{G}_{g,T_1},
\end{cases}
\qquad t\in\mathcal{T}_i.
\end{equation}
The low-temperature group contributes to the group statistics $(\mu_g, \sigma_g)$ but receives zero token advantage gradient; optimization is restricted to the high-temperature group. This asymmetric update rule ensures that the exploration-gain signal drives updates on high-temperature trajectories only: allowing the low-temperature gradient updates would introduce a competing update direction, weakening the exploration-gain correction carried by the high-temperature group (proof in Appendix~\ref{app:low_temp_group_update}).

\paragraph{Objective function.}
Policy optimization proceeds with a clipped importance-ratio objective. The ratio accounts for the discrepancy between the current distribution and the rollout distribution, while clipping constrains large updates to stabilize training:
\begin{equation}
\label{eq:tgrl_ratio}
r_{i,t}(\theta)
\coloneqq
\frac{\pi_\theta^{T_1}(y_{i,t}\mid x_g,y_{i,<t})}
{\pi_{\theta_{\mathrm{old}}}^{T_1}(y_{i,t}\mid x_g,y_{i,<t})}.
\end{equation}
Finally, we optimize the average clipped surrogate over the high-temperature rollouts,
\begin{equation}
\label{eq:tgrl_actor_obj}
\mathcal{L}_{\mathrm{TGRL}}(\theta)
=
-
\mathbb{E}_{g}\!\left[
\frac{1}{|\mathcal{G}_{g,T_1}|}
\sum_{i\in\mathcal{G}_{g,T_1}}
\frac{1}{N_i}
\sum_{t\in\mathcal{T}_i}
\min\!\Bigl(
r_{i,t}(\theta)\,\widetilde A_{i,t},
\mathrm{clip}\!\bigl(r_{i,t}(\theta),1-\epsilon_c,1+\epsilon_c\bigr)\,\widetilde A_{i,t}
\Bigr)
\right],
\end{equation}
where $\epsilon_c$ is the clipping range. The objective implements the TGRL principle: the temperature split measures prompt-level exploration gain, and token-level JS allocates that gain as credit to positions where the temperature intervention most strongly changes the decoding distribution.

\begin{table}[h]
    \centering
    \caption{Avg@16 results on mathematical reasoning benchmarks with Qwen3 variants \citep{yang2025qwen3} trained in \texttt{think} mode. \textbf{Best} results are in \textbf{bold}, second-best results are \underline{underlined}.}
    \label{tab:math1}
    \resizebox{0.98\textwidth}{!}{
    \begin{tabular}{l|cccccc|c}
        \toprule
        \textbf{Method} & \textbf{AIME24} & \textbf{AIME25} & \textbf{AMC23} & \textbf{MATH500} & \textbf{Minerva} & \textbf{Olympiad} & \textbf{Average} \\
        \midrule
        Qwen3-14B & 50.4 & 40.0 & 79.8 & 87.9 & 46.1 & 60.3 & 60.8 \\
        PPO & 55.2 & 45.8 & 90.6 & 93.9 & 49.1 & 63.5 & 66.4 \\
        GRPO & 57.5 & 45.4 & \textbf{92.9} & \underline{94.4} & \underline{49.7} & 62.2 & 67.0 \\
        DAPO & 61.0 & 46.0 & \underline{92.8} & 94.1 & \textbf{49.8} & \underline{64.8} & \underline{68.0} \\
        Dr.GRPO & \underline{61.2} & 44.5 & 91.7 & 94.2 & 49.2 & 64.5 & 67.5 \\
        RLOO & 59.8 & \underline{47.9} & 91.2 & 93.9 & 48.6 & 62.3 & 67.3 \\
        \textbf{TGRL (ours)} & \textbf{63.4} & \textbf{49.4} & 92.7 & \textbf{95.3} & 48.9 & \textbf{66.4} & \textbf{69.4} \\
                \midrule
        Qwen3-32B & 53.1 & 39.6 & 81.1 & 90.8 & 47.7 & 64.1 & 62.7 \\
        PPO & 54.6 & 43.8 & 80.0 & 91.6 & 48.9 & \underline{65.7} & 64.1 \\
        GRPO & 60.4 & 44.8 & 88.9 & 92.2 & 48.5 & 62.1 & 66.2 \\
        DAPO & \underline{61.6} & \underline{51.2} & 91.2 & 94.5 & 49.4 & 63.5 & \underline{68.6} \\
        Dr.GRPO & 55.6 & 47.9 & 90.3 & 94.0 & \underline{49.8} & 65.5 & 67.2 \\
        RLOO & 61.0 & 45.0 & \underline{91.3} & \underline{94.6} & \underline{49.8} & 63.7 & 67.5 \\
        \textbf{TGRL (ours)} & \textbf{62.3} & \textbf{51.6} & \textbf{91.4} & \textbf{95.2} & \textbf{50.6} & \textbf{69.9} & \textbf{70.2} \\
        \bottomrule
    \end{tabular}}
    \vspace{-0.5em}
\end{table}

\begin{table}[h]
\centering
\caption{Performance on code-generation benchmarks using Qwen3-4B. We report LiveCodeBench Avg@16 and Pass@16, CodeForces Rating and Percentile, and HumanEval+ Pass@16.}
\label{code_1}
\resizebox{0.92\textwidth}{!}{
\begin{tabular}{l|cc|cc|c}
\toprule
\multirow{2}{*}{\textbf{Method}} 
& \multicolumn{2}{c|}{\textbf{LiveCodeBench}} 
& \multicolumn{2}{c|}{\textbf{CodeForces}} 
& \multicolumn{1}{c}{\textbf{HumanEval+}} \\
& \textbf{Avg@16} & \textbf{Pass@16} & \textbf{Rating} & \textbf{Percentile} &  \textbf{Pass@16} \\
\midrule
Qwen3-4B & 30.5 & 40.9 & 578.8 & 1.2 & 89.0 \\
PPO      &   \score{39.7}{(+9.2)}   &  \score{\underline{57.7}}{(+16.8)}    &  \score{\underline{1377.6}}{(+798.8)}  &   \score{\underline{71.9}}{(+70.7)}  &   \score{95.1}{(+6.1)}   \\
GRPO     & \score{39.5}{(+9.0)} & \score{55.1}{(+14.2)} & \score{1267.9}{(+689.1)} & \score{63.1}{(+61.9)} & \score{95.7}{(+6.7)} \\
DAPO     & \score{41.0}{(+10.5)} & \score{52.3}{(+11.4)} & \score{1112.5}{(+533.7)} & \score{46.7}{(+45.5)} & \score{95.7}{(+6.7)} \\
Dr.GRPO  & \score{41.6}{(+11.1)} & \score{\underline{57.7}}{(+16.8)} & \score{1195.0}{(+616.2)} & \score{55.3}{(+54.1)} & \score{95.7}{(+6.7)} \\
RLOO     & \score{\underline{43.2}}{(+12.7)} & \score{56.9}{(+16.0)} & \score{1224.0}{(+645.2)} & \score{58.6}{(+57.4)} & \score{\textbf{97.5}}{(+8.5)} \\
\midrule
\textbf{TGRL (ours)} & \score{\textbf{43.6}}{(+13.1)} & \score{\textbf{62.1}}{(+21.2)} & \score{\textbf{1574.3}}{(+995.5)} & \score{\textbf{84.3}}{(+83.1)} & \score{\underline{96.9}}{(+7.9)} \\
\bottomrule
\end{tabular}}
\vspace{-0.5em}
\end{table}

\section{Experiments}
\subsection{Experimental Setup}
\label{sec:experiments}

\paragraph{Baselines.}
We evaluate TGRL on a challenging suite spanning mathematical reasoning, code generation, and long-horizon agentic tasks. This suite probes whether RLVR methods can convert limited rollout budgets into effective exploration on complex tasks, rather than merely optimizing saturated benchmarks. Our baseline pool covers representative policy-gradient, group-relative, leave-one-out, temperature-adaptive, and agentic planning methods, including PPO \citep{ouyang2022training}, GRPO \citep{guo2025deepseek}, DAPO \citep{yu2025dapo}, Dr.GRPO \citep{liu2025drgrpo}, RLOO \citep{ahmadian2024back}, TAMPO \citep{dang2026temperature}, ReAct \citep{yao2022react}, and EMPG \citep{wang2025harnessing}. We instantiate TGRL on Qwen3-4B, Qwen3-14B and Qwen3-32B for mathematics, Qwen3-4B for code generation, and Qwen2.5-7B-Instruct for agentic tasks.

\paragraph{Benchmarks.}
The suite covers a difficulty gradient and three task structures. For mathematical reasoning, we use six benchmarks: AIME 2024/2025~\citep{AIME} and AMC 2023~\citep{AMC} (competition-level, small problem counts), MATH500~\citep{lightman2024lets} (balanced cross-difficulty subset), Minerva~\citep{minerva} (quantitative STEM reasoning), and Olympiad~\citep{he2024olympiadbench} (olympiad-level multi-step problems). For code generation, we use LiveCodeBench~\citep{jain2024livecodebench} (contamination-free, periodically refreshed), CodeForces~\citep{penedo2025codeforces} (Elo-rated against human contestants), and HumanEval+~\citep{chen2021humanevaluating,liu2023your} (HumanEval with hardened test cases for functional correctness). Math and code evaluations use decoding temperature 0.6, top-$p$ 0.95, and a maximum response length of 8192 tokens. For long-horizon agentic tasks, we follow~\citep{wang2025harnessing} on ALFWorld (multi-step household task completion) and WebShop (sparse-reward web shopping). Together, these benchmarks probe whether TGRL's prompt-level exploration signal generalizes across difficulty tiers, code paradigms, and sparse-reward multi-turn decision making.

\paragraph{Training.}
For mathematics, we train on DeepScaleR~\citep{deepscaler2025}; for code, we train on DeepCoder~\citep{deepcoder2025}; and for agent tasks, we use WebShop~\citep{yao2022webshop} and ALFWorld~\citep{shridhar2020alfworld}. Math and code experiments are conducted in Qwen3 think mode with a maximum response length of 8192 tokens. Following \citep{dang2026temperature}, we use a 20-step warmup in which all trajectories are sampled at high temperature and optimized with the standard single-temperature group-normalized advantage. This prevents the low-temperature group from becoming a degenerate baseline before the policy develops stable reasoning patterns. All methods use the rollout budget of $4$ responses. TGRL allocates one low-temperature sample at $T_0{=}0.3$ to $\mathcal{G}_{g,T_0}$ and three samples to $\mathcal{G}_{g,T_1}$, reserving most of the budget for exploratory rollouts. Motivated by the role separation in Lemma~\ref{lem:dual_var}, $T_0{=}0.3$ serves as a conservative reference; Table~\ref{tab:split_sensitivity} further evaluates the effect of different low-/high-temperature rollout splits.

\begin{table}[h]
\centering
\caption{Performance on ALFWorld and WebShop averaged over three seeds on Qwen2.5-7B-Instruct. We report Success Rate for ALFWorld and both Task Score and Success Rate for WebShop.}
\label{tab:alfworld}
\resizebox{\textwidth}{!}{
\begin{tabular}{l|ccccccc|cc}
\toprule
\multirow{2}{*}{\textbf{Method}} & \multicolumn{7}{c}{\textbf{ALFWorld}} & \multicolumn{2}{c}{\textbf{WebShop}} \\
\cmidrule(lr){2-8} \cmidrule(lr){9-10}
 & \textbf{Pick} & \textbf{Look} & \textbf{Clean} & \textbf{Heat} & \textbf{Cool} & \textbf{Pick2} & \textbf{All} & \textbf{Task Score} & \textbf{Succ.} \\
\midrule
GPT-4o \citep{hurst2024gpt} & 75.3 & 60.8 & 31.2 & 56.7 & 21.6 & 49.8 & 48.0 & 31.8 & 23.7 \\
DeepSeek-V4-Pro & 69.0 & 66.7 & 41.9 & 26.3 & 52.2 & 60.0 & 51.6 & 77.9 & 33.3 \\
\midrule
Prompting Backbone & 33.4 & 21.6 & 19.3 & 6.9 & 2.8 & 3.2 & 14.8 & 26.4 & 7.8 \\
Prompting ReAct & 48.5 & 35.4 & 34.3 & 13.2 & 18.2 & 17.6 & 31.2 & 46.2 & 19.5 \\
PPO & 92.3 & 64.0 & 92.5 & 89.5 & 80.3 & 68.8 & \underline{80.4} & \underline{81.4} & 68.7 \\
GRPO & 88.8 & 43.7 & 88.1 & 70.3 & 77.7 & 56.8 & 74.8 & 77.8 & 65.6 \\
EMPG & 92.9 & 75.2 & 74.8 & 86.3 & 73.7 & 65.3 & 78.5 & 81.0 & \underline{69.3} \\
\textbf{TGRL (ours)} & 97.0 & 69.2 & 96.7 & 71.4 & 66.7 & 95.0 & \textbf{86.7} & \textbf{85.7} & \textbf{74.2} \\

\bottomrule
\end{tabular}
}
\vspace{-1.0em}
\end{table}

\subsection{Main Results}

Tables~\ref{tab:math1}, \ref{code_1}, and \ref{tab:alfworld} report results on the full evaluation suite, covering mathematical reasoning, code generation, and long-horizon agent tasks under the evaluation protocols described above. Overall, TGRL broadly improves over strong RLVR baselines across domains, with the clearest gains on benchmarks that require complex reasoning. On mathematical reasoning, TGRL achieves the best six-benchmark average at both Qwen3-14B and Qwen3-32B, reaching $69.4\%$ and $70.2\%$ and outperforming the strongest baseline by $1.4\%$ and $1.6\%$, respectively. The gains are especially clear on challenging competition-style benchmarks: at 14B, TGRL improves over the best prior result by $2.2\%$ on AIME24, $1.5\%$ on AIME25, and $1.6\%$ on Olympiad; at 32B, it is best on all six benchmarks, with the largest margin on Olympiad ($+4.2\%$) and an additional gain on AIME25 ($+0.4\%$). Since these competition benchmarks contain relatively few test problems, such as 30 problems in each AIME set, we further conduct 5-seed Avg@16 evaluations on AIME24, AIME25, and AMC23 for both TGRL and GRPO at both scales (Appendix~\ref{app:multi_seed}). The cross-seed variance remains modest, and the mean gaps on AIME are comparable to or larger than the per-seed standard deviations, supporting the robustness of the reported improvements. We also compare against exploration-oriented RLVR baselines in Table~\ref{tab:entropy_baselines_14b}, including entropy-regularized GRPO and TAMPO, a temperature-adaptive method with globally scheduled decoding temperature. Neither family closes the gap to TGRL at 14B ($69.4\%$ vs.\ $67.8\%$ for the best entropy-regularized run and $68.0\%$ for TAMPO), and supplementary Qwen3-4B results show the same ordering in Table~\ref{tab:math_4b_methods}.

The same pattern extends to code generation and long-horizon agent tasks. On code benchmarks, TGRL obtains the best LiveCodeBench Avg@16 and Pass@16, improving Pass@16 by $4.4\%$ over the strongest prior result, and raises CodeForces rating from $1377.6$ to $1574.3$ ($+196.7$), corresponding to a percentile gain from $71.9\%$ to $84.3\%$. These improvements are most pronounced on LiveCodeBench and CodeForces, where success depends on exploring alternative algorithms, implementation choices, and debugging trajectories. On the near-saturated HumanEval+, TGRL remains competitive with the top baselines. On long-horizon agent tasks, TGRL reaches $86.7\%$ overall success on ALFWorld ($+6.3\%$ over the best RL baseline) and $74.2\%$ success on WebShop ($+4.9\%$), while also improving WebShop task score to $85.7\%$. The largest ALFWorld sub-task gain appears on \textit{Pick2} ($95.0\%$ vs.\ $68.8\%$ for PPO), which requires early commitment to a multi-step plan under sparse feedback. Across domains, the strongest gains occur on challenging competition problems, unsaturated code benchmarks, and sparse-reward long-horizon tasks, consistent with TGRL's design: mixed-temperature grouping estimates when broader exploration is beneficial for a prompt, and JS-based credit allocation directs this signal to the tokens most affected by the temperature intervention.

\subsection{Training Dynamics Analysis}
\label{sec:mechanism_analysis}

Following \citep{liu2025drgrpo,yu2025dapo,wang2025ragen}, Figure~\ref{fig:training_dynamics_14b} reports diagnostics that test whether TGRL realizes its two design goals: prompt-level exploration-gain estimation and token-level credit allocation. On AIME24/AIME25, increasing GRPO's sampling temperature improves performance from 54.8\%/37.5\% to 60.0\%/44.2\%, but still trails TGRL's 61.9\%/49.6\% and yields longer, more truncated responses. Thus, higher rollout stochasticity helps but is not sufficient on its own: TGRL instead uses a low-/high-temperature reward gap as an empirical estimate of prompt-level exploration gain.

\begin{figure}[t]
    \centering
    \includegraphics[width=\textwidth]{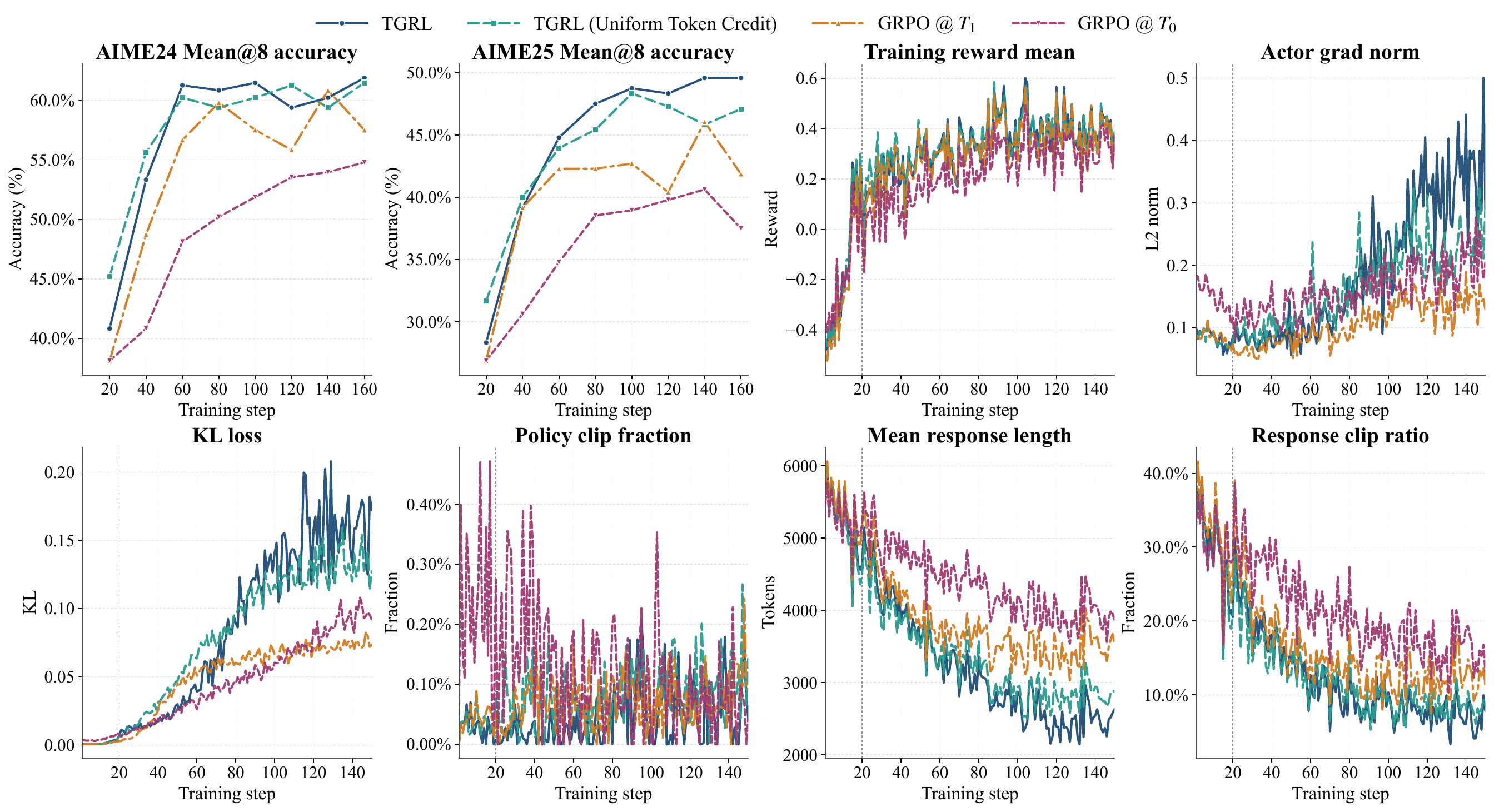}
    \caption{Training dynamics on the Qwen3-14B run. Compared with fixed-temperature GRPO variants, TGRL achieves higher downstream accuracy while producing shorter responses and lower truncation. The uniform-token-credit ablation helps isolate the effect of JS-based credit allocation.}
    \label{fig:training_dynamics_14b}
    \vspace{-1.0em}
\end{figure}

The update statistics provide a complementary diagnostic. During steps 121--150, TGRL attains the largest gradient norm (0.331) and policy KL (0.160), while keeping the PPO clip fraction below 0.11\%. Since KL measures distributional movement \citep{bai2022training} and clip fraction counts how often token ratios enter the clipped PPO region \citep{wang2019trust}, this combination indicates larger effective policy updates without widespread ratio clipping. This is consistent with JS-based credit allocation, which concentrates the prompt-level exploration signal on positions where the low- and high-temperature policies differ, rather than spreading it uniformly across all tokens.

The response-length curves show a similar pattern. Uniform token credit spreads the signal over the entire response, whereas JS-based credit allocation concentrates more of it on temperature-sensitive positions. TGRL produces the shortest late-stage responses, averaging 2476 tokens versus 3915 for GRPO @ $T_0$ and 3482 for GRPO @ $T_1$, and the lowest response clipping ratio, 7.2\% versus 16.0\% and 12.4\%. TGRL and TGRL (Uniform Token Credit) also have nearly identical late-stage training rewards, 0.394 versus 0.388, yet differ in downstream accuracy, response length, and truncation, suggesting that the gain comes from token-level credit allocation rather than reward scale. Together, these dynamics support the proposed mechanism: mixed-temperature grouping estimates prompt-level exploration gain, and JS-based credit allocation directs it to temperature-sensitive tokens.

\begin{table}[h]
\centering
\caption{Exploration-oriented baselines using Qwen3-14B. ``GRPO+Ent'' denotes GRPO with an added entropy bonus of coefficient $\lambda$; TAMPO is a temperature-adaptive baseline.}
\label{tab:entropy_baselines_14b}
\resizebox{0.95\textwidth}{!}{
\begin{tabular}{l|cccccc|c}
\toprule
\textbf{Method} & \textbf{AIME24} & \textbf{AIME25} & \textbf{AMC23} & \textbf{MATH500} & \textbf{Minerva} & \textbf{Olympiad} & \textbf{Average} \\
\midrule
GRPO & 57.5 & 45.4 & 92.9 & 94.4 & 49.7 & 62.2 & 67.0 \\
GRPO+Ent ($\lambda{=}10^{-4}$) & 63.2 & 48.4 & 90.5 & 93.7 & 47.9 & 62.8 & 67.8 \\
GRPO+Ent ($\lambda{=}10^{-3}$) & 57.1 & 47.7 & 91.9 & 92.8 & 47.0 & 63.8 & 66.7 \\
TAMPO & 63.3 & 47.1 & 92.0 & 95.2 & 49.2 & 61.0 & 68.0 \\
TGRL (ours) & 63.4 & 49.4 & 92.7 & 95.3 & 48.9 & 66.4 & \textbf{69.4} \\
\bottomrule
\end{tabular}}
   \vspace{-0.5em}
\end{table}

\begin{table}[h]
    \centering
    \caption{
    Math ablation results. Comparing GRPO @ $T_0$ and GRPO @ $T_1$ isolates single-temperature exploration. Comparing GRPO @ $T_1$ and TGRL (Uniform Token Credit) evaluates the mixed-temperature TGRL structure without JS-based credit allocation. Comparing TGRL (Uniform Token Credit) and TGRL isolates JS-based token credit allocation.
    }
    \label{tab:ablation_qwen3_32b}
    \resizebox{\textwidth}{!}{
    \begin{tabular}{c l|cccccc|c}
        \toprule
        \textbf{Scale} & \textbf{Method} & \textbf{AIME24} & \textbf{AIME25} & \textbf{AMC23} & \textbf{MATH500} & \textbf{Minerva} & \textbf{Olympiad} & \textbf{Average} \\
        \midrule
        \multirow{4}{*}{14B}
        & GRPO @ $T_0$ & 57.0 & 44.6 & 91.9 & 93.6 & 48.9 & 63.0 & 66.5 \\
        & GRPO @ $T_1$ & 57.5 & 45.4 & 92.9 & 94.4 & 49.7 & 62.2 & 67.0 \\
        & TGRL (Uniform Token Credit) & 61.0 & 46.9 & 91.3 & 93.4 & 48.8 & 68.1 & 68.2 \\
        & \textbf{TGRL (ours)} & 63.4 & 49.4 & 92.7 & 95.3 & 48.9 & 66.4 & \textbf{69.4} \\
        \midrule
        \multirow{4}{*}{32B}
        & GRPO @ $T_0$ & 55.8 & 44.4 & 85.6 & 93.0 & 48.0 & 64.6 & 65.1 \\
        & GRPO @ $T_1$ & 60.4 & 44.8 & 88.9 & 92.2 & 48.5 & 62.1 & 66.2 \\
        & TGRL (Uniform Token Credit) & 58.3 & 48.1 & 90.3 & 94.9 & 50.2 & 64.7 & 67.8 \\
        & \textbf{TGRL (ours)} & 62.3 & 51.6 & 91.4 & 95.2 & 50.6 & 69.9 & \textbf{70.2} \\
        \bottomrule
    \end{tabular}}
    \vspace{-0.8em}
\end{table}

\subsection{Ablation Studies}
\label{sec:ablations}

\paragraph{Mechanism ablations.}
Table~\ref{tab:ablation_qwen3_32b} shows the same decomposition after full training and across scales. Raising the temperature alone is limited: GRPO @ $T_1$ improves average accuracy over GRPO @ $T_0$ by only $0.5\%$ at 14B and $1.1\%$ at 32B. Replacing single-temperature GRPO with mixed-temperature grouping gives the first clear gain, raising the average from $67.0\%$ to $68.2\%$ at 14B and from $66.2\%$ to $67.8\%$ at 32B. Adding JS-based credit allocation then further raises the average to $69.4\%$ and $70.2\%$, with especially large additional gains on AIME25 (+$2.5\%$ at 14B; +$3.5\%$ at 32B). The benchmark-level changes are not uniform, which is expected: JS weighting does not create exploration by itself, but redistributes the grouped signal toward temperature-sensitive tokens. Together with Figure~\ref{fig:training_dynamics_14b}, these results support a two-stage interpretation: mixed-temperature grouping creates a prompt-level exploration signal, and JS-based credit allocation converts that signal into more focused token-level credit.

\paragraph{Sensitivity to the exploration temperature.}
We next study how TGRL's performance varies with the exploration temperature. Figure~\ref{fig:temp_sweep_4b_diagnostics} and Table~\ref{tab:temp_sweep_4b_summary} show that the key requirement is a nontrivial temperature gap: performance degrades only when the gap is too narrow ($T_1{=}0.8$), whereas a broad regime of larger gaps ($T_1\in[1.0,1.6]$) remains competitive. Within this regime, $T_1{=}1.2$ achieves the best efficiency--stability trade-off, reaching strong validation performance earlier than higher-temperature alternatives while maintaining smooth optimization dynamics. We therefore adopt $T_1{=}1.2$ as the default in the main experiments.

\paragraph{Training efficiency analysis.}
Head-to-head wall-clock comparisons against GRPO and DAPO are reported in Table~\ref{tab:wallclock}. Within a matched 12-hour budget, TGRL attains an AIME-Combined score of $1.0833$ versus GRPO's $1.0354$ ($+0.048$ absolute); against DAPO, TGRL reaches an equivalent accuracy target $36\%$ faster ($6.4$ h vs.\ $10.0$ h). These results show that fine-grained exploration-gain estimation improves per-step update quality, enabling TGRL to learn more efficiently under the same compute budget.

\section{Conclusion}
This work demonstrates that sampling temperature can serve not merely as a decoding parameter, but as a means to construct an exploration signal for RLVR. TGRL forms low-temperature reference and high-temperature exploration groups for each prompt, estimates exploration gain from their reward gap, and allocates this gain to temperature-sensitive tokens via JS divergence. Across mathematical reasoning, code generation, and long-horizon agent tasks, the results validate this design: TGRL reaches $70.2\%$ average math accuracy on Qwen3-32B, improves CodeForces rating by $196.7$ points, and achieves $86.7\%$ success on ALFWorld. Ablations and training dynamics further show that mixed-temperature grouping and JS-based credit allocation provide complementary benefits, confirming that the key is not simply sampling more stochastically, but extracting the value of exploration within each prompt and assigning it to the tokens where the temperature intervention matters.
More broadly, TGRL establishes a principle for connecting quantified exploration gain to token-level policy-gradient credit, providing a design pattern for efficient exploration in RLVR.
Future work may extend this principle to adaptive multi-temperature grouping, richer token-credit allocation criteria, and more open-ended interactive environments.

\newpage
\bibliographystyle{unsrt}
\bibliography{cite}

\newpage
\appendix

\section{Supplementary Theory and Proofs}
\label{app:proofs}

This appendix expands the derivations that are stated more compactly in the main text. It includes the detailed proofs for prompt-level exploration gain under mixed-temperature grouping and the JS-based local-sensitivity result used for token-level credit allocation, together with several equivalent forms and implementation-level connections used in the discussion. A later section records the design rationale for using the low-temperature group in the mixed-temperature statistics while excluding it from direct actor updates.

\subsection{Why Uniform-temperature Group Normalization Cannot Recover Exploration Gain}
\label{app:grpo_gap}

Suppose all $|\mathcal{G}_{g,T_0+T_1}|$ trajectories in a prompt group are sampled at a single temperature $T$, as in standard uniform-temperature group optimization. The group-normalized advantage then takes the form
\begin{equation}
\label{eq:uniform_temp_adv}
    A_i^{(T)}
    \coloneqq
    \frac{R_i-\bar{R}_g^{(T)}}{\sqrt{(\sigma_g^{(T)})^2+\epsilon}},
    \qquad
    \bar{R}_g^{(T)} \coloneqq \frac{1}{|\mathcal{G}_{g,T_0+T_1}|}\sum_{i\in\mathcal{G}_{g,T_0+T_1}}R_i.
\end{equation}
This statistic can rank samples relative to the group mean, but it has no in-group low-/high-temperature gap. In particular, there is no quantity analogous to
\[
\Delta R_g = \bar{R}_g^{T_1} - \bar{R}_g^{T_0},
\]
because the group contains only one temperature-conditioned response cloud. Consequently, uniform-temperature GRPO can exploit within-temperature variability, but it cannot directly estimate whether moving from conservative decoding to exploratory decoding helps on the same prompt. TGRL creates this missing quantity by explicitly splitting each group into a low-temperature group at $T_0$ and a high-temperature group at $T_1$.

\subsection{Proof of Lemma~\ref{lem:dual_var} (Exploration Gain under Mixed-Temperature Grouping)}
\label{app:proof_dual_var}

\begin{proof}
Fix a prompt $x$ and drop the group subscript for readability. Let $\{R_i^{T_0}\}_{i\in\mathcal{G}_{T_0}}$ be i.i.d. rewards drawn from $\pi_\theta^{T_0}(\cdot\mid x)$, with mean $\mu_{T_0}(x)$ and variance $\sigma_{T_0}^2(x)$. Let $\{R_i^{T_1}\}_{i\in\mathcal{G}_{T_1}}$ be i.i.d. rewards drawn from $\pi_\theta^{T_1}(\cdot\mid x)$, with mean $\mu_{T_1}(x)$ and variance $\sigma_{T_1}^2(x)$. Assume the low- and high-temperature samples are mutually independent.

Define
\begin{equation}
    \bar{R}^{T_0} \coloneqq \frac{1}{|\mathcal{G}_{T_0}|}\sum_{i\in\mathcal{G}_{T_0}} R_i^{T_0},
    \qquad
    \bar{R}^{T_1} \coloneqq \frac{1}{|\mathcal{G}_{T_1}|}\sum_{i\in\mathcal{G}_{T_1}} R_i^{T_1},
    \qquad
    \Delta R \coloneqq \bar{R}^{T_1} - \bar{R}^{T_0}.
\end{equation}
By linearity of expectation,
\begin{equation}
    \mathbb{E}[\bar{R}^{T_0} \mid x] = \mu_{T_0}(x),
    \qquad
    \mathbb{E}[\bar{R}^{T_1} \mid x] = \mu_{T_1}(x),
\end{equation}
so
\begin{equation}
    \mathbb{E}[\Delta R \mid x]
    = \mathbb{E}[\bar{R}^{T_1} \mid x] - \mathbb{E}[\bar{R}^{T_0} \mid x]
    = \mu_{T_1}(x) - \mu_{T_0}(x).
\end{equation}

For the variance, independence of the two branches yields
\begin{align}
    \mathrm{Var}(\Delta R \mid x)
    &= \mathrm{Var}(\bar{R}^{T_1} - \bar{R}^{T_0} \mid x) \nonumber \\
    &= \mathrm{Var}(\bar{R}^{T_1} \mid x) + \mathrm{Var}(\bar{R}^{T_0} \mid x)
       - 2\underbrace{\mathrm{Cov}(\bar{R}^{T_1}, \bar{R}^{T_0} \mid x)}_{=\,0}.
\end{align}
Since each branch is an average of i.i.d. samples,
\begin{equation}
    \mathrm{Var}(\bar{R}^{T_1} \mid x) = \frac{\sigma_{T_1}^2(x)}{|\mathcal{G}_{T_1}|},
    \qquad
    \mathrm{Var}(\bar{R}^{T_0} \mid x) = \frac{\sigma_{T_0}^2(x)}{|\mathcal{G}_{T_0}|}.
\end{equation}
Substituting gives
\begin{equation}
    \mathrm{Var}(\Delta R \mid x) = \frac{\sigma_{T_1}^2(x)}{|\mathcal{G}_{T_1}|} + \frac{\sigma_{T_0}^2(x)}{|\mathcal{G}_{T_0}|},
\end{equation}
which proves Eq.~\eqref{eq:dual_var}.
\end{proof}

\paragraph{Interpretation.}
Lemma~\ref{lem:dual_var} shows that the low-temperature group provides the within-prompt reference needed to estimate prompt-specific exploration gain. These per-group gains later enter the mixed-group advantage defined in Section~\ref{sec:algorithm}. The next proposition explains how mixed-group normalization injects this exploration gain into the high-temperature-group learning signal.

\subsection{Proof of Proposition~\ref{prop:mixed_adv} (Mixed-group Advantage Decomposition)}
\label{app:proof_mixed_adv}

\begin{proof}
Fix a prompt group $g$ with low-temperature set $\mathcal{G}_{g,T_0}$, high-temperature set $\mathcal{G}_{g,T_1}$, and total set $\mathcal{G}_{g,T_0+T_1}\coloneqq\mathcal{G}_{g,T_0}\cup\mathcal{G}_{g,T_1}$. Let
\begin{equation}
    \bar{R}_g^{T_0} \coloneqq \frac{1}{|\mathcal{G}_{g,T_0}|}\sum_{i\in\mathcal{G}_{g,T_0}} R_i,
    \qquad
    \bar{R}_g^{T_1} \coloneqq \frac{1}{|\mathcal{G}_{g,T_1}|}\sum_{i\in\mathcal{G}_{g,T_1}} R_i,
    \qquad
    \Delta R_g \coloneqq \bar{R}_g^{T_1} - \bar{R}_g^{T_0}.
\end{equation}
The mixed-group mean is
\begin{equation}
\label{eq:app_group_mean_repeat}
    \mu_g = \frac{1}{|\mathcal{G}_{g,T_0+T_1}|}\sum_{i\in \mathcal{G}_{g,T_0+T_1}} R_i
    = \frac{|\mathcal{G}_{g,T_0}|\bar{R}_g^{T_0} + |\mathcal{G}_{g,T_1}|\bar{R}_g^{T_1}}{|\mathcal{G}_{g,T_0+T_1}|}.
\end{equation}
We consider the mixed-group advantage
\begin{equation}
    A_i \coloneqq \frac{R_i-\mu_g}{\sqrt{\sigma_g^2+\epsilon}},
    \qquad
    \sigma_g^2 \coloneqq \frac{1}{|\mathcal{G}_{g,T_0+T_1}|}\sum_{i\in \mathcal{G}_{g,T_0+T_1}}(R_i-\mu_g)^2,
\end{equation}
where $\epsilon>0$ is a numerical stabilizer.
The goal is to decompose $A_i$ into a within-subgroup fluctuation term plus a cross-group exploration-gain term.

\noindent\textbf{Step 1: Two equivalent forms of the mixed-group mean.}
Starting from \eqref{eq:app_group_mean_repeat}, substitute $\bar{R}_g^{T_1} = \bar{R}_g^{T_0} + \Delta R_g$ to obtain
\begin{align}
    \mu_g
    &= \frac{|\mathcal{G}_{g,T_0}|\bar{R}_g^{T_0} + |\mathcal{G}_{g,T_1}|(\bar{R}_g^{T_0} + \Delta R_g)}{|\mathcal{G}_{g,T_0+T_1}|} \nonumber\\
    &= \bar{R}_g^{T_0} + \frac{|\mathcal{G}_{g,T_1}|}{|\mathcal{G}_{g,T_0+T_1}|}\,\Delta R_g.
    \label{eq:app_mean_form_baseline}
\end{align}
Likewise, substitute $\bar{R}_g^{T_0} = \bar{R}_g^{T_1} - \Delta R_g$ to obtain
\begin{align}
    \mu_g
    &= \frac{|\mathcal{G}_{g,T_0}|(\bar{R}_g^{T_1} - \Delta R_g) + |\mathcal{G}_{g,T_1}|\bar{R}_g^{T_1}}{|\mathcal{G}_{g,T_0+T_1}|} \nonumber\\
    &= \bar{R}_g^{T_1} - \frac{|\mathcal{G}_{g,T_0}|}{|\mathcal{G}_{g,T_0+T_1}|}\,\Delta R_g.
    \label{eq:app_mean_form_explore}
\end{align}
Thus the mixed-group mean can be written relative to either subgroup mean, with the offset determined by the exploration gain.

\noindent\textbf{Step 2: Decomposition for low-temperature-group samples.}
Take any $i\in\mathcal{G}_{g,T_0}$. Using \eqref{eq:app_mean_form_baseline},
\begin{align}
    R_i - \mu_g
    &= R_i - \left(\bar{R}_g^{T_0} + \frac{|\mathcal{G}_{g,T_1}|}{|\mathcal{G}_{g,T_0+T_1}|}\,\Delta R_g\right) \nonumber\\
    &= (R_i - \bar{R}_g^{T_0}) - \frac{|\mathcal{G}_{g,T_1}|}{|\mathcal{G}_{g,T_0+T_1}|}\,\Delta R_g.
    \label{eq:app_baseline_decomp_raw}
\end{align}
Dividing by $\sqrt{\sigma_g^2 + \epsilon}$ gives
\begin{equation}
\label{eq:app_baseline_decomp_norm}
    A_i
    =
    \frac{R_i-\bar{R}_g^{T_0}}{\sqrt{\sigma_g^2+\epsilon}}
    - \frac{|\mathcal{G}_{g,T_1}|}{|\mathcal{G}_{g,T_0+T_1}|}\,\frac{\Delta R_g}{\sqrt{\sigma_g^2+\epsilon}},
    \qquad i\in\mathcal{G}_{g,T_0}.
\end{equation}
This is exactly the low-temperature-group case of Eq.~\eqref{eq:mixed_group_decomp}.

\noindent\textbf{Step 3: Decomposition for high-temperature-group samples.}
Take any $i\in\mathcal{G}_{g,T_1}$. Using \eqref{eq:app_mean_form_explore},
\begin{align}
    R_i - \mu_g
    &= R_i - \left(\bar{R}_g^{T_1} - \frac{|\mathcal{G}_{g,T_0}|}{|\mathcal{G}_{g,T_0+T_1}|}\,\Delta R_g\right) \nonumber\\
    &= (R_i - \bar{R}_g^{T_1}) + \frac{|\mathcal{G}_{g,T_0}|}{|\mathcal{G}_{g,T_0+T_1}|}\,\Delta R_g.
    \label{eq:app_explore_decomp_raw}
\end{align}
Dividing by $\sqrt{\sigma_g^2 + \epsilon}$ gives
\begin{equation}
\label{eq:app_explore_decomp_norm}
    A_i
    =
    \frac{R_i-\bar{R}_g^{T_1}}{\sqrt{\sigma_g^2+\epsilon}}
    + \frac{|\mathcal{G}_{g,T_0}|}{|\mathcal{G}_{g,T_0+T_1}|}\,\frac{\Delta R_g}{\sqrt{\sigma_g^2+\epsilon}},
    \qquad i\in\mathcal{G}_{g,T_1}.
\end{equation}
This is exactly the high-temperature-group case of Eq.~\eqref{eq:mixed_group_decomp}.

\noindent\textbf{Step 4: Sign interpretation of the cross-group shift.}
When $\Delta R_g > 0$, the additive term in \eqref{eq:app_explore_decomp_norm} is positive, so the mixed-group advantage shifts every high-temperature-group sample upward relative to its within-subgroup deviation. At the same time, the additive term in \eqref{eq:app_baseline_decomp_norm} is negative, so every low-temperature-group sample is shifted downward. When $\Delta R_g < 0$, the signs reverse. Hence the mixed-group advantage is jointly determined by a within-subgroup fluctuation term and a cross-group gap term whose sign is controlled by the prompt-level exploration gain.

This proves Proposition~\ref{prop:mixed_adv}.
\end{proof}

\subsection{Subgroup-mean Identity}
\label{app:subgroup_means}

Averaging Eq.~\eqref{eq:mixed_group_decomp} over each subgroup gives
\begin{equation}
\label{eq:sign_carrying_subgroup_means}
    \bar{A}_g^{T_0}
    =
    -\frac{|\mathcal{G}_{g,T_1}|}{|\mathcal{G}_{g,T_0+T_1}|}\frac{\Delta R_g}{\sqrt{\sigma_g^2+\epsilon}},
    \qquad
    \bar{A}_g^{T_1}
    =
    \frac{|\mathcal{G}_{g,T_0}|}{|\mathcal{G}_{g,T_0+T_1}|}\frac{\Delta R_g}{\sqrt{\sigma_g^2+\epsilon}}.
\end{equation}

\begin{proof}
Average the two cases of Eq.~\eqref{eq:mixed_group_decomp} separately over $\mathcal{G}_{g,T_0}$ and $\mathcal{G}_{g,T_1}$. For the low-temperature group,
\begin{align}
    \bar A_g^{T_0}
    &= \frac{1}{|\mathcal{G}_{g,T_0}|}\sum_{i\in\mathcal{G}_{g,T_0}} A_i \\
    &= \frac{1}{|\mathcal{G}_{g,T_0}|}\sum_{i\in\mathcal{G}_{g,T_0}}\left(\frac{R_i-\bar R_g^{T_0}}{\sqrt{\sigma_g^2+\epsilon}} - \frac{|\mathcal{G}_{g,T_1}|}{|\mathcal{G}_{g,T_0+T_1}|}\frac{\Delta R_g}{\sqrt{\sigma_g^2+\epsilon}}\right) \nonumber\\
    &= \frac{1}{\sqrt{\sigma_g^2+\epsilon}}\left(\frac{1}{|\mathcal{G}_{g,T_0}|}\sum_{i\in\mathcal{G}_{g,T_0}}(R_i-\bar R_g^{T_0})\right) - \frac{|\mathcal{G}_{g,T_1}|}{|\mathcal{G}_{g,T_0+T_1}|}\frac{\Delta R_g}{\sqrt{\sigma_g^2+\epsilon}}.
\end{align}
The bracketed term is zero because $\bar R_g^{T_0}$ is the empirical low-temperature-group mean, hence
\begin{equation}
    \bar A_g^{T_0} = -\frac{|\mathcal{G}_{g,T_1}|}{|\mathcal{G}_{g,T_0+T_1}|}\frac{\Delta R_g}{\sqrt{\sigma_g^2+\epsilon}}.
\end{equation}
For the high-temperature group,
\begin{align}
    \bar A_g^{T_1}
    &= \frac{1}{|\mathcal{G}_{g,T_1}|}\sum_{i\in\mathcal{G}_{g,T_1}} A_i \\
    &= \frac{1}{|\mathcal{G}_{g,T_1}|}\sum_{i\in\mathcal{G}_{g,T_1}}\left(\frac{R_i-\bar R_g^{T_1}}{\sqrt{\sigma_g^2+\epsilon}} + \frac{|\mathcal{G}_{g,T_0}|}{|\mathcal{G}_{g,T_0+T_1}|}\frac{\Delta R_g}{\sqrt{\sigma_g^2+\epsilon}}\right) \nonumber\\
    &= \frac{1}{\sqrt{\sigma_g^2+\epsilon}}\left(\frac{1}{|\mathcal{G}_{g,T_1}|}\sum_{i\in\mathcal{G}_{g,T_1}}(R_i-\bar R_g^{T_1})\right) + \frac{|\mathcal{G}_{g,T_0}|}{|\mathcal{G}_{g,T_0+T_1}|}\frac{\Delta R_g}{\sqrt{\sigma_g^2+\epsilon}}.
\end{align}
Again the bracketed term vanishes by definition of $\bar R_g^{T_1}$, so
\begin{equation}
    \bar A_g^{T_1} = \frac{|\mathcal{G}_{g,T_0}|}{|\mathcal{G}_{g,T_0+T_1}|}\frac{\Delta R_g}{\sqrt{\sigma_g^2+\epsilon}}.
\end{equation}
This proves Eq.~\eqref{eq:sign_carrying_subgroup_means}.
\end{proof}

\subsection{Proof of Lemma~\ref{lem:js_local} (Global temperature-response control of token-level JS)}
\label{app:proof_js_local}

\begin{proof}
We fix a trajectory $i$ and position $t$, and drop subscripts for readability. Let $z \in \mathbb{R}^{|V|}$ denote the logit vector, define the inverse-temperature parameter $\beta \coloneqq 1/T$, and write
\begin{equation}
    p_a^{(\beta)} \coloneqq \frac{\exp(\beta z_a)}{\sum_{b=1}^{|V|}\exp(\beta z_b)}.
\end{equation}
Set $\beta_0 \coloneqq 1/T_0$, $\beta_1 \coloneqq 1/T_1$, and $\Delta\beta \coloneqq \beta_1-\beta_0$. Define the scalar log-partition function
\begin{equation}
    A(\beta) \coloneqq \log\!\left(\sum_{a=1}^{|V|}e^{\beta z_a}\right).
\end{equation}
Then
\begin{equation}
    A'(\beta)=\mathbb{E}_{a\sim p^{(\beta)}}[z_a],
    \qquad
    A''(\beta)=\mathrm{Var}_{a\sim p^{(\beta)}}[z_a].
\end{equation}

\noindent\textbf{Step 1: Exact Fisher-path identity for the Jeffreys divergence.}
For any $\beta,\beta'$, the KL divergence between the corresponding categorical exponential-family members is
\begin{align}
    \mathrm{KL}\!\left(p^{(\beta')}\middle\|p^{(\beta)}\right)
    &= \sum_{a=1}^{|V|} p_a^{(\beta')}
       \log\!\frac{p_a^{(\beta')}}{p_a^{(\beta)}} \nonumber\\
    &= \sum_{a=1}^{|V|} p_a^{(\beta')}
       \Bigl((\beta'-\beta)z_a - A(\beta') + A(\beta)\Bigr) \nonumber\\
    &= A(\beta)-A(\beta') + (\beta'-\beta)A'(\beta').
\end{align}
Therefore the Jeffreys divergence satisfies
\begin{align}
\label{eq:app_jeffreys_gap}
    D_{\mathrm{J}}\!\left(p^{(\beta_0)},p^{(\beta_1)}\right)
    &\coloneqq
    \mathrm{KL}\!\left(p^{(\beta_0)}\middle\|p^{(\beta_1)}\right)
    +
    \mathrm{KL}\!\left(p^{(\beta_1)}\middle\|p^{(\beta_0)}\right) \nonumber\\
    &=
    (\beta_1-\beta_0)\bigl(A'(\beta_1)-A'(\beta_0)\bigr).
\end{align}
Applying the fundamental theorem of calculus to $A'$ and using $A''(\beta)=\mathrm{Var}_{a\sim p^{(\beta)}}[z_a]$ gives
\begin{align}
    A'(\beta_1)-A'(\beta_0)
    &= \Delta\beta \int_0^1 A''(\beta_0+s\Delta\beta)\,ds \nonumber\\
    &= \Delta\beta \int_0^1 \mathrm{Var}_{a\sim p^{(\beta_0+s\Delta\beta)}}[z_a]\,ds.
\end{align}
Substituting into Eq.~\eqref{eq:app_jeffreys_gap} yields the exact path identity
\begin{equation}
\label{eq:app_jeffreys_path}
    D_{\mathrm{J}}\!\left(p^{(\beta_0)},p^{(\beta_1)}\right)
    =
    (\Delta\beta)^2
    \int_0^1 \mathrm{Var}_{a\sim p^{(\beta_0+s\Delta\beta)}}[z_a]\,ds.
\end{equation}

\noindent\textbf{Step 2: JS is upper-bounded by one eighth of the Jeffreys divergence.}
The Jensen--Shannon divergence is the $f$-divergence with generator
\begin{equation}
    f_{\mathrm{JS}}(u) = \tfrac{u}{2}\log\!\tfrac{2u}{1+u} + \tfrac{1}{2}\log\!\tfrac{2}{1+u},
\end{equation}
while the Jeffreys divergence is the $f$-divergence with generator
\begin{equation}
    f_{\mathrm{J}}(u) = (u-1)\log u.
\end{equation}
Define
\begin{equation}
    g(u) \coloneqq f_{\mathrm{J}}(u) - 8 f_{\mathrm{JS}}(u).
\end{equation}
Direct differentiation gives
\begin{equation}
    g(1)=0,
    \qquad
    g'(1)=0,
    \qquad
    g''(u)=\frac{(u-1)^2}{u^2(1+u)} \ge 0
    \quad \text{for all } u>0.
\end{equation}
Hence $g$ is convex and attains its global minimum at $u=1$, so
\begin{equation}
\label{eq:app_pointwise_js_jeffreys}
    8 f_{\mathrm{JS}}(u) \le f_{\mathrm{J}}(u)
    \qquad \text{for all } u>0.
\end{equation}
Integrating Eq.~\eqref{eq:app_pointwise_js_jeffreys} pointwise over the reference measure gives
\begin{equation}
\label{eq:app_js_jeffreys}
    \mathrm{JS}\!\left(p^{(\beta_0)},p^{(\beta_1)}\right)
    \le
    \frac{1}{8}
    D_{\mathrm{J}}\!\left(p^{(\beta_0)},p^{(\beta_1)}\right).
\end{equation}

\noindent\textbf{Step 3: Combine the two identities.}
Substituting Eq.~\eqref{eq:app_jeffreys_path} into Eq.~\eqref{eq:app_js_jeffreys} yields
\begin{equation}
    \mathrm{JS}\!\left(p^{(\beta_0)},p^{(\beta_1)}\right)
    \le
    \frac{(\Delta\beta)^2}{8}
    \int_0^1 \mathrm{Var}_{a\sim p^{(\beta_0+s\Delta\beta)}}[z_a]\,ds.
\end{equation}
Restoring the trajectory and position indices gives Eq.~\eqref{eq:js_expansion}.

\paragraph{Mutual-information view.}
Let $C\sim \mathrm{Unif}\{0,1\}$ and, conditional on $C=c$, sample $A$ from $p^{(\beta_c)}$. The marginal distribution of $A$ is $m=\tfrac{1}{2}(p^{(\beta_0)}+p^{(\beta_1)})$. Therefore
\begin{align}
    I(C;A)
    &= H(A)-H(A\mid C) \nonumber\\
    &= H(m)-\tfrac{1}{2}H\!\left(p^{(\beta_0)}\right)-\tfrac{1}{2}H\!\left(p^{(\beta_1)}\right) \nonumber\\
    &= \tfrac{1}{2}\mathrm{KL}\!\left(p^{(\beta_0)}\middle\|m\right)
      + \tfrac{1}{2}\mathrm{KL}\!\left(p^{(\beta_1)}\middle\|m\right) \nonumber\\
    &= \mathrm{JS}\!\left(p^{(\beta_0)},p^{(\beta_1)}\right),
\end{align}
which is the identity used in the main-text interpretation.
\end{proof}


\subsection{Local-cooling Probe for JS-based Credit Allocation}
\label{app:js_allocation_appendix}

Proposition~\ref{prop:credit_quality} is stated for one-step continuation values: a larger $J_{i,t}$ yields a larger certified budget $b_{i,t}=\sqrt{2J_{i,t}}$ on how much the low-/high-temperature intervention can change downstream value at token $t$. To relate this quantity to realized span-level behavior under the decoding procedure used by TGRL, we evaluate the same directional intervention directly on generated trajectories through a \emph{local-cooling} probe.

The protocol is as follows. Starting from successful exploration trajectories sampled at the high temperature $T_1$, we (i)~select a contiguous response span of fixed token budget according to a \emph{selector} criterion, (ii)~replace each token in that span with the token preferred under the low-temperature distribution $p^{(T_0)}$, and (iii)~continue autoregressive generation from the edited prefix. Two damage metrics are recorded: the \emph{flip rate}, defined as the fraction of originally successful trajectories that become unsuccessful after the edit, and the \emph{mean score drop}, defined as the average decrease in the scalar verifier reward.

Under this probe, selectors that target spans with larger temperature sensitivity are expected to induce larger verifier damage under the cooling intervention.

\begin{table}[t]
    \centering
    \caption{
    Local-cooling probe on the Qwen3-32B model. Each selector selects a contiguous span under a fixed edit budget within successful exploration trajectories; the span is cooled toward the low-temperature preference and generation continues from the edited prefix. \textsc{Matched-Random} is repeated five times per source trajectory, hence its larger count. Higher flip rate ($\uparrow$) and larger score drop ($\uparrow$) indicate greater temperature sensitivity of the selected span.
    }
    \label{tab:js_local_cooling_32b}
    \resizebox{0.98\linewidth}{!}{
    \begin{tabular}{lccccc}
        \toprule
        \textbf{Selector} & \textbf{Count} & \textbf{Flip rate $\uparrow$} & \textbf{Mean score drop $\uparrow$} & \textbf{Mean position ratio} & \textbf{Mean JS} \\
        \midrule
        \textsc{JS} & 224 & \textbf{12.1\%} & \textbf{0.241} & 0.444 & \textbf{0.0427} \\
        \textsc{Low-Margin} & 224 & 11.6\% & 0.232 & 0.482 & 0.0309 \\
        \textsc{Matched-Random} & 1120 & 8.7\% & 0.173 & 0.458 & 0.0064 \\
        \textsc{Entropy} & 224 & 7.6\% & 0.152 & 0.468 & 0.0385 \\
        \bottomrule
    \end{tabular}}
\end{table}

Table~\ref{tab:js_local_cooling_32b} reports the results on the Qwen3-32B model. Three findings are noteworthy.

\emph{(i)~JS spans produce the largest realized damage under cooling.}
The \textsc{JS} selector achieves the highest flip rate (\textbf{12.1\%}) and the largest mean score drop (\textbf{0.241}), dominating \textsc{Matched-Random} (8.7\%, 0.173), \textsc{Entropy} (7.6\%, 0.152), and \textsc{Low-Margin} (11.6\%, 0.232). Since the mean position ratios are comparable across selectors (0.444--0.482), this ordering is not explained by coarse positional confounds.

\emph{(ii)~Temperature sensitivity is distinct from generic uncertainty.}
The \textsc{Entropy} selector targets high-entropy positions, while \textsc{Low-Margin} targets positions where the top-two logits are closest. Both are natural proxies for token-level uncertainty, yet neither matches the damage induced by \textsc{JS}. This indicates that temperature-gap JS---the local sensitivity of the next-token distribution to the low/high-temperature perturbation---captures information not contained in these scalar uncertainty measures.

\emph{(iii)~JS-targeted spans align with larger realized damage.}
The ordering $\textsc{JS} > \textsc{Low-Margin} > \textsc{Matched-Random} > \textsc{Entropy}$ in both damage metrics shows that selectors targeting larger-JS spans also produce larger downstream damage under the explicit cooling intervention. In particular, the \textsc{JS} selector also exhibits the highest mean JS among the spans it selects (0.0427 vs.\ 0.0064--0.0385 for other selectors), indicating that it most directly targets positions where the temperature switch changes the next-token distribution.

\paragraph{Interpretation.}
The probe applies a discrete span-level version of the temperature-switch mechanism studied in Proposition~\ref{prop:credit_quality} and measures the resulting verifier degradation. The observed damage ordering is consistent with the role of JS in directing credit toward temperature-sensitive positions under the decoding regime used by TGRL.

\subsection{Proof of Proposition~\ref{prop:credit_quality} (Credit Allocation toward Temperature-Sensitive Positions)}
\label{app:proof_credit}

\begin{proof}
Fix a high-temperature trajectory $i$ and write
\begin{equation}
    J_t \coloneqq J_{i,t},
    \qquad
    b_t \coloneqq \sqrt{2J_t},
    \qquad
    s_t \coloneqq h_i(J_t),
    \qquad
    q_t \coloneqq \frac{s_t}{\sum_{r\in\mathcal{T}_i}s_r},
\end{equation}
where $t\in\mathcal{T}_i$ and $N_i\coloneqq |\mathcal{T}_i|$. Let $\tau$ be drawn uniformly from $\mathcal{T}_i$. All expectations and covariances below are taken with respect to this uniformly drawn token index, conditional on the fixed trajectory.

\noindent\textbf{Step 1: Bound the local temperature-intervention effect.}
Fix any position $t\in\mathcal{T}_i$ and any bounded continuation-value map $v_t\colon \{1,\dots,|V|\}\to[0,1]$. Operationally, $v_t(a)$ can be read as the normalized downstream value of choosing token $a$ at position $t$ and then continuing the rollout. The variational characterization of total variation gives
\begin{equation}
\label{eq:app_tv_dual}
    \left|
    \mathbb{E}_{a\sim p_{i,t}^{(1)}}[v_t(a)]
    -
    \mathbb{E}_{a\sim p_{i,t}^{(0)}}[v_t(a)]
    \right|
    \le
    \mathrm{TV}(p_{i,t}^{(0)},p_{i,t}^{(1)}).
\end{equation}
Let $m_t \coloneqq \tfrac{1}{2}(p_{i,t}^{(0)}+p_{i,t}^{(1)})$. Since
\begin{equation}
    \mathrm{TV}(p_{i,t}^{(0)},m_t)=\mathrm{TV}(p_{i,t}^{(1)},m_t)=\frac{1}{2}\mathrm{TV}(p_{i,t}^{(0)},p_{i,t}^{(1)}),
\end{equation}
Pinsker's inequality yields
\begin{align}
    J_t
    &= \frac{1}{2}\mathrm{KL}(p_{i,t}^{(0)}\|m_t) + \frac{1}{2}\mathrm{KL}(p_{i,t}^{(1)}\|m_t) \nonumber\\
    &\ge \frac{1}{2}\cdot 2\,\mathrm{TV}(p_{i,t}^{(0)},m_t)^2
      + \frac{1}{2}\cdot 2\,\mathrm{TV}(p_{i,t}^{(1)},m_t)^2 \nonumber\\
    &= \frac{1}{2}\,\mathrm{TV}(p_{i,t}^{(0)},p_{i,t}^{(1)})^2.
\end{align}
Therefore
\begin{equation}
    \mathrm{TV}(p_{i,t}^{(0)},p_{i,t}^{(1)}) \le \sqrt{2J_t}=b_t.
\end{equation}
Combining this with Eq.~\eqref{eq:app_tv_dual} proves Eq.~\eqref{eq:local_effect_bound}.

\noindent\textbf{Step 2: Rewrite the weighted average and compare it with uniform broadcasting.}
By definition of $q_t$,
\begin{equation}
\label{eq:app_quality_ratio_1}
    \sum_{t\in\mathcal{T}_i} q_t b_t
    = \frac{\sum_{t\in\mathcal{T}_i} s_t b_t}{\sum_{r\in\mathcal{T}_i}s_r}
    = \frac{\mathbb{E}_{\tau}[s_\tau b_\tau]}{\mathbb{E}_{\tau}[s_\tau]}.
\end{equation}
The corresponding uniform allocation is
\begin{equation}
\label{eq:app_uniform_quality}
    \frac{1}{N_i}\sum_{t\in\mathcal{T}_i} b_t
    = \mathbb{E}_{\tau}[b_\tau].
\end{equation}
Using $\mathbb{E}[XY]=\mathbb{E}[X]\mathbb{E}[Y]+\mathrm{Cov}(X,Y)$ with $X=s_\tau$ and $Y=b_\tau$,
\begin{equation}
    \mathbb{E}_{\tau}[s_\tau b_\tau]
    = \mathbb{E}_{\tau}[s_\tau] \, \mathbb{E}_{\tau}[b_\tau]
    + \mathrm{Cov}_{\tau}(s_\tau,b_\tau).
\end{equation}
Substituting into Eq.~\eqref{eq:app_quality_ratio_1} gives
\begin{equation}
\label{eq:app_credit_difference_main}
    \sum_{t\in\mathcal{T}_i} q_t b_t
    - \frac{1}{N_i}\sum_{t\in\mathcal{T}_i} b_t
    = \frac{\mathrm{Cov}_{\tau}(s_\tau,b_\tau)}{\mathbb{E}_{\tau}[s_\tau]}.
\end{equation}
This is the covariance form underlying Eq.~\eqref{eq:theory_weight_allocation}.

\noindent\textbf{Step 3: Non-negativity of the covariance.}
Both $s_t=h_i(J_t)$ and $b_t=\sqrt{2J_t}$ are non-decreasing functions of $J_t$. By the Chebyshev's order inequality in Lemma~\ref{lem:chebyshev_order},
\begin{equation}
    \mathrm{Cov}_{\tau}(s_\tau,b_\tau) \ge 0.
\end{equation}
Substituting this into Eq.~\eqref{eq:app_credit_difference_main} yields
\begin{equation}
    \sum_{t\in\mathcal{T}_i} q_t b_t
    \ge
    \frac{1}{N_i}\sum_{t\in\mathcal{T}_i} b_t,
\end{equation}
which is exactly Eq.~\eqref{eq:theory_weight_allocation}.

\noindent\textbf{Step 4: Ratio form.}
Whenever the uniform average is positive, Eq.~\eqref{eq:app_credit_difference_main} can also be written as
\begin{equation}
\label{eq:credit_gain_ratio}
\frac{\sum_{t\in\mathcal{T}_i} q_t b_t}{\frac{1}{N_i}\sum_{t\in\mathcal{T}_i} b_t}
=
1+
\frac{\mathrm{Cov}_{\tau}(s_\tau,b_\tau)}{\mathbb{E}_{\tau}[s_\tau]\,\mathbb{E}_{\tau}[b_\tau]}
=
1+\mathrm{Corr}_{\tau}(s_\tau,b_\tau)\,\mathrm{CV}_{\tau}(s_\tau)\,\mathrm{CV}_{\tau}(b_\tau)
\ge 1,
\end{equation}
where $\mathrm{CV}_{\tau}(X)\coloneqq \mathrm{Std}_{\tau}(X)/\mathbb{E}_{\tau}[X]$. This is the ratio form referenced in the main text.

\noindent\textbf{Step 5: Connection to the implemented weight function.}
In the implementation, for a fixed high-temperature trajectory $i$ we define
\begin{equation}
\begin{aligned}
\bar{J}_i &\coloneqq \frac{1}{N_i}\sum_{t\in\mathcal{T}_i}J_{i,t},
\qquad
\rho_{i,t} \coloneqq \frac{J_{i,t}+\epsilon}{\bar{J}_i+\epsilon},
\qquad
\omega_{i,t} \coloneqq \log\!\left(1 + \rho_{i,t}\right), \\
\bar{\omega}_i &\coloneqq \frac{1}{N_i}\sum_{t\in\mathcal{T}_i}\omega_{i,t},
\qquad
w_{i,t} \coloneqq \frac{\omega_{i,t}}{\bar{\omega}_i}.
\end{aligned}
\end{equation}
Once the trajectory is fixed, both $\bar{J}_i$ and $\bar{\omega}_i$ are constants, so the implemented score is a trajectory-wise non-decreasing function of $J_{i,t}$. The final division by $\bar{\omega}_i$ rescales all token scores by the same positive constant, so it does not change the induced profile. Therefore Proposition~\ref{prop:credit_quality} applies directly to the normalized weights in Eq.~\eqref{eq:tgrl_js_mean}. In addition, by construction,
\begin{equation}
    \frac{1}{N_i}\sum_{t\in\mathcal{T}_i} w_{i,t} = 1,
    \qquad
    \frac{1}{N_i}\sum_{t\in\mathcal{T}_i} \widetilde A_{i,t} = A_i.
\end{equation}
So the practical weight function preserves the trajectory-average scale of the rollout-level advantage while redistributing it across token positions.
\end{proof}

\subsection{Co-monotonicity and Chebyshev's Order Inequality}
\label{app:comonotone}

\begin{lemma}[Chebyshev's order inequality {\citep[Section~2.17]{hardy1952inequalities}}]
\label{lem:chebyshev_order}
Let $X$ be a random variable with $\mathbb{E}[g(X)^2] < \infty$ and $\mathbb{E}[f(X)^2] < \infty$, and let $g, f\colon \mathbb{R} \to \mathbb{R}$ be two non-decreasing measurable functions. Then $\mathrm{Cov}(g(X), f(X)) \ge 0$.
\end{lemma}

\begin{proof}
Let $X'$ be an independent copy of $X$. We claim the identity
\begin{equation}
\label{eq:app_cov_identity}
    2\,\mathrm{Cov}(g(X), f(X))
    = \mathbb{E}\bigl[(g(X) - g(X'))(f(X) - f(X'))\bigr].
\end{equation}
To verify, expand the right-hand side:
\begin{align}
    &\mathbb{E}\bigl[(g(X) - g(X'))(f(X) - f(X'))\bigr] \nonumber \\
    &\quad= \mathbb{E}[g(X)f(X)] - \mathbb{E}[g(X)f(X')] - \mathbb{E}[g(X')f(X)] + \mathbb{E}[g(X')f(X')].
    \label{eq:app_cov_expand_detail}
\end{align}
Since $X$ and $X'$ are independent, the cross terms factor: $\mathbb{E}[g(X)f(X')] = \mathbb{E}[g(X)]\,\mathbb{E}[f(X')]$. Since $X \stackrel{d}{=} X'$, we have $\mathbb{E}[g(X')] = \mathbb{E}[g(X)]$, $\mathbb{E}[f(X')] = \mathbb{E}[f(X)]$, and $\mathbb{E}[g(X')f(X')] = \mathbb{E}[g(X)f(X)]$. Substituting,
\begin{align}
    \eqref{eq:app_cov_expand_detail}
    &= \mathbb{E}[g(X)f(X)] - \mathbb{E}[g(X)]\mathbb{E}[f(X)] - \mathbb{E}[g(X)]\mathbb{E}[f(X)] + \mathbb{E}[g(X)f(X)] \nonumber \\
    &= 2\bigl(\mathbb{E}[g(X)f(X)] - \mathbb{E}[g(X)]\mathbb{E}[f(X)]\bigr)
    = 2\,\mathrm{Cov}(g(X), f(X)),
\end{align}
confirming \eqref{eq:app_cov_identity}. Now, since both $g$ and $f$ are non-decreasing, for any realization $(x, x')$: if $x \ge x'$ then $g(x) \ge g(x')$ and $f(x) \ge f(x')$; if $x \le x'$ then $g(x) \le g(x')$ and $f(x) \le f(x')$. In both cases, the product $(g(x)-g(x'))(f(x)-f(x')) \ge 0$. Therefore the integrand in \eqref{eq:app_cov_identity} is non-negative almost surely, and $\mathrm{Cov}(g(X), f(X)) \ge 0$.
\end{proof}

\paragraph{Application.}
In Proposition~\ref{prop:credit_quality}, the scalar variable is $X = Z_t = J_{i,t}$, the first monotone function is $g(X)=h_i(X)$, and the second monotone function is $f(X)=\sqrt{2X}$. The lemma therefore gives $\mathrm{Cov}_{\tau}(h_i(Z_t), \sqrt{2Z_t}) \ge 0$, which is the key ingredient in Step~3 of Appendix~\ref{app:proof_credit}.

\subsection{Attribution Control for Excluding the Low-temperature Group from Direct Actor Updates}
\label{app:low_temp_group_update}

Eq.~\eqref{eq:tgrl_token_adv} sets $\widetilde A_{i,t}=0$ for every
low-temperature-group token. This is a deliberate attribution-control design:
the low-temperature group is used as a within-prompt reference for estimating
exploration gain, but it is not treated as a direct actor-update target. This
section makes the distinction precise.

\paragraph{Stop-gradient convention.}
As in PPO/GRPO-style policy-gradient optimization, sampled responses, verifier
rewards, group statistics, and token weights are treated as fixed quantities
during one actor-update step. In particular, $(R_i,\mu_g,\sigma_g,A_i,J_{i,t},
w_{i,t})$ are detached from the actor gradient. Under this convention, the
low-temperature samples influence the update only through the mixed-group
statistics $(\mu_g,\sigma_g)$ and the reward contrast $\Delta R_g$, not through
token-level score-function terms.

For a low-temperature rollout $j\in\mathcal{G}_{g,T_0}$, even if one writes the
corresponding temperature-$T_0$ ratio
\begin{equation}
    r^{(0)}_{j,t}(\theta)
    \coloneqq
    \frac{\pi_\theta^{T_0}(y_{j,t}\mid x_g,y_{j,<t})}
    {\pi_{\theta_{\mathrm{old}}}^{T_0}(y_{j,t}\mid x_g,y_{j,<t})},
\end{equation}
the token advantage is exactly zero. Hence the clipped token contribution is
identically zero:
\begin{equation}
\label{eq:app_low_zero_exact}
\min\!\Bigl(
r^{(0)}_{j,t}(\theta)\cdot 0,\,
\mathrm{clip}\!\bigl(r^{(0)}_{j,t}(\theta),1-\epsilon_c,1+\epsilon_c\bigr)\cdot 0
\Bigr)
=0.
\end{equation}
Therefore no low-temperature token can receive a direct actor-gradient
attribution. The low-temperature branch is not discarded, however: its rewards
enter $\bar R_g^{T_0}$, $\mu_g$, $\sigma_g$, and $\Delta R_g$, and these
quantities shift the high-temperature advantages through
Proposition~\ref{prop:mixed_adv}.

\paragraph{Branch-wise decomposition of the exploration-gain carrier.}
Fix a prompt group $g$ and define
\begin{equation}
    n_0 \coloneqq |\mathcal{G}_{g,T_0}|,\qquad
    n_1 \coloneqq |\mathcal{G}_{g,T_1}|,\qquad
    n \coloneqq n_0+n_1,\qquad
    s_g \coloneqq \sqrt{\sigma_g^2+\epsilon},\qquad
    \widehat\gamma_g \coloneqq \frac{\Delta R_g}{s_g}.
\end{equation}
For compactness, write the within-branch centered rewards as
\begin{equation}
    B_i^{T_1}
    \coloneqq
    \frac{R_i-\bar R_g^{T_1}}{s_g},
    \quad i\in\mathcal{G}_{g,T_1},
    \qquad
    B_j^{T_0}
    \coloneqq
    \frac{R_j-\bar R_g^{T_0}}{s_g},
    \quad j\in\mathcal{G}_{g,T_0}.
\end{equation}
Then Eq.~\eqref{eq:mixed_group_decomp} can be rewritten as
\begin{equation}
\label{eq:app_branch_adv_decomp}
    A_i = B_i^{T_1} + \frac{n_0}{n}\widehat\gamma_g,
    \quad i\in\mathcal{G}_{g,T_1},
    \qquad
    A_j = B_j^{T_0} - \frac{n_1}{n}\widehat\gamma_g,
    \quad j\in\mathcal{G}_{g,T_0}.
\end{equation}
Thus the same prompt-level exploration-gain estimate $\widehat\gamma_g$ appears
with opposite signs in the two branches.

For a high-temperature rollout $i\in\mathcal{G}_{g,T_1}$, define its
JS-weighted score vector
\begin{equation}
\label{eq:app_high_score_vector}
    \psi_i^{T_1}(\theta)
    \coloneqq
    \frac{1}{N_i}\sum_{t\in\mathcal{T}_i}
    w_{i,t}\,
    \nabla_\theta
    \log \pi_\theta^{T_1}(y_{i,t}\mid x_g,y_{i,<t}),
    \qquad
    \bar\psi_g^{T_1}
    \coloneqq
    \frac{1}{n_1}\sum_{i\in\mathcal{G}_{g,T_1}}\psi_i^{T_1}(\theta).
\end{equation}
For comparison, consider a hypothetical update-both variant that also updates
low-temperature rollouts with some unit-mean low-branch token weights
$v_{j,t}$, where $N_j^{-1}\sum_{t\in\mathcal{T}_j}v_{j,t}=1$. Its low-branch
score vector would be
\begin{equation}
\label{eq:app_low_score_vector}
    \psi_j^{T_0}(\theta)
    \coloneqq
    \frac{1}{N_j}\sum_{t\in\mathcal{T}_j}
    v_{j,t}\,
    \nabla_\theta
    \log \pi_\theta^{T_0}(y_{j,t}\mid x_g,y_{j,<t}),
    \qquad
    \bar\psi_g^{T_0}
    \coloneqq
    \frac{1}{n_0}\sum_{j\in\mathcal{G}_{g,T_0}}\psi_j^{T_0}(\theta).
\end{equation}

\begin{proposition}[Exploration-gain attribution under high-temperature-only updates]
\label{prop:low_temp_attribution_control}
Ignoring clipping only for the purpose of decomposing the first-order
score-function carrier, the high-temperature-only TGRL ascent carrier for prompt
group $g$ is
\begin{equation}
\label{eq:app_high_only_carrier}
    \mathcal{U}_g^{\mathrm{TGRL}}
    \coloneqq
    \frac{1}{n_1}
    \sum_{i\in\mathcal{G}_{g,T_1}}
    A_i\,\psi_i^{T_1}(\theta)
    =
    \underbrace{
    \frac{1}{n_1}
    \sum_{i\in\mathcal{G}_{g,T_1}}
    B_i^{T_1}\,\psi_i^{T_1}(\theta)
    }_{\text{within-high-temperature ranking}}
    +
    \underbrace{
    \frac{n_0}{n}\widehat\gamma_g\,\bar\psi_g^{T_1}
    }_{\text{exploration-gain carrier}} .
\end{equation}
By contrast, a symmetric update-both variant that averages over both branches
would have the carrier
\begin{equation}
\label{eq:app_update_both_carrier}
\begin{aligned}
    \mathcal{U}_g^{\mathrm{both}}
    &\coloneqq
    \frac{1}{n}
    \left(
    \sum_{i\in\mathcal{G}_{g,T_1}} A_i\,\psi_i^{T_1}(\theta)
    +
    \sum_{j\in\mathcal{G}_{g,T_0}} A_j\,\psi_j^{T_0}(\theta)
    \right) \\
    &=
    \underbrace{
    \frac{1}{n}
    \left(
    \sum_{i\in\mathcal{G}_{g,T_1}} B_i^{T_1}\,\psi_i^{T_1}(\theta)
    +
    \sum_{j\in\mathcal{G}_{g,T_0}} B_j^{T_0}\,\psi_j^{T_0}(\theta)
    \right)
    }_{\text{within-branch ranking terms}} \\
    &\quad+
    \underbrace{
    \frac{n_0n_1}{n^2}\widehat\gamma_g
    \left(
    \bar\psi_g^{T_1}-\bar\psi_g^{T_0}
    \right)
    }_{\text{contrastive high-minus-low exploration-gain carrier}} .
\end{aligned}
\end{equation}
Therefore, updating both branches changes the exploration-gain carrier from a
pure high-temperature attribution term into a contrastive score-difference term
that includes an explicit negative low-branch component.
\end{proposition}

\begin{proof}
Substitute Eq.~\eqref{eq:app_branch_adv_decomp} into the high-temperature-only
carrier:
\begin{align}
    \mathcal{U}_g^{\mathrm{TGRL}}
    &=
    \frac{1}{n_1}
    \sum_{i\in\mathcal{G}_{g,T_1}}
    \left(B_i^{T_1}+\frac{n_0}{n}\widehat\gamma_g\right)
    \psi_i^{T_1}(\theta) \nonumber \\
    &=
    \frac{1}{n_1}
    \sum_{i\in\mathcal{G}_{g,T_1}}
    B_i^{T_1}\,\psi_i^{T_1}(\theta)
    +
    \frac{n_0}{n}\widehat\gamma_g
    \left(
    \frac{1}{n_1}
    \sum_{i\in\mathcal{G}_{g,T_1}}\psi_i^{T_1}(\theta)
    \right),
\end{align}
which proves Eq.~\eqref{eq:app_high_only_carrier}. For the update-both variant,
substituting Eq.~\eqref{eq:app_branch_adv_decomp} gives
\begin{align}
    \mathcal{U}_g^{\mathrm{both}}
    &=
    \frac{1}{n}
    \sum_{i\in\mathcal{G}_{g,T_1}}
    \left(B_i^{T_1}+\frac{n_0}{n}\widehat\gamma_g\right)
    \psi_i^{T_1}(\theta)
    +
    \frac{1}{n}
    \sum_{j\in\mathcal{G}_{g,T_0}}
    \left(B_j^{T_0}-\frac{n_1}{n}\widehat\gamma_g\right)
    \psi_j^{T_0}(\theta) \nonumber \\
    &=
    \frac{1}{n}
    \left(
    \sum_{i\in\mathcal{G}_{g,T_1}} B_i^{T_1}\,\psi_i^{T_1}(\theta)
    +
    \sum_{j\in\mathcal{G}_{g,T_0}} B_j^{T_0}\,\psi_j^{T_0}(\theta)
    \right) \nonumber \\
    &\quad+
    \frac{n_0}{n^2}\widehat\gamma_g
    \sum_{i\in\mathcal{G}_{g,T_1}}\psi_i^{T_1}(\theta)
    -
    \frac{n_1}{n^2}\widehat\gamma_g
    \sum_{j\in\mathcal{G}_{g,T_0}}\psi_j^{T_0}(\theta) \nonumber \\
    &=
    \frac{1}{n}
    \left(
    \sum_{i\in\mathcal{G}_{g,T_1}} B_i^{T_1}\,\psi_i^{T_1}(\theta)
    +
    \sum_{j\in\mathcal{G}_{g,T_0}} B_j^{T_0}\,\psi_j^{T_0}(\theta)
    \right)
    +
    \frac{n_0n_1}{n^2}\widehat\gamma_g
    \left(
    \bar\psi_g^{T_1}-\bar\psi_g^{T_0}
    \right),
\end{align}
which proves Eq.~\eqref{eq:app_update_both_carrier}.
\end{proof}

\paragraph{Interpretation.}
When $\Delta R_g>0$, the high-temperature branch has empirically outperformed
the low-temperature reference on prompt $x_g$. Eq.~\eqref{eq:app_high_only_carrier}
then assigns the exploration-gain carrier to the high-temperature score direction
$\bar\psi_g^{T_1}$, after JS-based token weighting has  allocated the trajectory signal across token positions. If the low branch were also updated, Eq.~\eqref{eq:app_update_both_carrier}
would simultaneously add a negative term $-\bar\psi_g^{T_0}$. This negative
mirror term is not evidence about which high-temperature token caused the
successful exploratory trajectory; it is a consequence of using the low branch
as the opposite side of the reward contrast. Conversely, when $\Delta R_g<0$,
TGRL suppresses unsuccessful high-temperature exploration through the high-branch
carrier, without turning the conservative reference samples into an additional
positive training target.

Thus, the zero-gradient rule implements a clean separation of roles:
\begin{equation}
\label{eq:app_attribution_path}
\mathcal{G}_{g,T_0}
\;\longrightarrow\;
(\bar R_g^{T_0},\mu_g,\sigma_g,\Delta R_g)
\;\longrightarrow\;
A_i\ \text{for}\ i\in\mathcal{G}_{g,T_1}
\;\longrightarrow\;
\widetilde A_{i,t}\,
\nabla_\theta\log\pi_\theta^{T_1}(y_{i,t}\mid x_g,y_{i,<t}),
\end{equation}
while blocking the direct path
\begin{equation}
\label{eq:app_blocked_low_path}
\mathcal{G}_{g,T_0}
\;\not\longrightarrow\;
\widetilde A_{j,t}\,
\nabla_\theta\log\pi_\theta^{T_0}(y_{j,t}\mid x_g,y_{j,<t}),
\qquad
j\in\mathcal{G}_{g,T_0}.
\end{equation}
The low-temperature samples therefore remain essential for estimating the
prompt-level exploration gain, but they do not receive token-level credit. This
prevents the conservative reference branch from becoming a competing actor-update
target and reduces cross-branch interference inside a clipped-ratio surrogate.

\paragraph{Budget accounting.}
This role separation also makes the budget accounting explicit. TGRL uses the
same rollout budget as the single-temperature baselines, but it deliberately
spends $n_0$ rollouts on estimating the low-temperature reference and $n_1$
rollouts on high-temperature actor updates. Under the default split
$(n_0,n_1)=(1,3)$, one response estimates the conservative reference and three
responses provide exploratory update samples. Appendix~\ref{app:time_complexity}
accounts for this choice at the systems level, and
Appendix~\ref{app:split_sensitivity} evaluates the low-/high-temperature split
under the same total rollout budget. Therefore the zero-gradient design should
be read as an attribution-control mechanism.

\clearpage

\section{Algorithm Design}
\label{app:algorithm}

Algorithm~\ref{alg:tgrl} summarizes the complete TGRL training procedure for one policy-update step, consisting of three phases: (i)~mixed-temperature rollout and grouped advantage computation, (ii)~JS-based token credit allocation on the high-temperature group, and (iii)~assembly of the clipped surrogate objective.

\begin{algorithm}[h]
\caption{Temperature-Grouped Reinforcement Learning (TGRL) --- one policy-update step}
\label{alg:tgrl}
\begin{algorithmic}[1]
\Require Policy $\pi_\theta$; prompt batch $\{x_g\}$; temperatures $T_0 < T_1$; subgroup sizes $|\mathcal{G}_{g,T_0}|$, $|\mathcal{G}_{g,T_1}|$; clip range $\epsilon_c$; stabilizer $\epsilon$
\Ensure Updated policy parameters $\theta$

\For{each prompt $x_g$ in the batch}

    \Statex \hspace{\algorithmicindent}\textcolor{gray}{\textit{\% Mixed-temperature rollout and grouped advantage}}
    \State Sample $|\mathcal{G}_{g,T_0}|$ low-temperature responses $\{y_i\}_{i\in\mathcal{G}_{g,T_0}}$ from $\pi_\theta^{T_0}(\cdot\mid x_g)$
    \State Sample $|\mathcal{G}_{g,T_1}|$ high-temperature responses $\{y_i\}_{i\in\mathcal{G}_{g,T_1}}$ from $\pi_\theta^{T_1}(\cdot\mid x_g)$
    \State Obtain rewards $\{R_i\}_{i\in\mathcal{G}_{g,T_0+T_1}}$ from the verifier, where $\mathcal{G}_{g,T_0+T_1} = \mathcal{G}_{g,T_0} \cup \mathcal{G}_{g,T_1}$
    \State Compute group statistics:
        $\mu_g \gets \frac{1}{|\mathcal{G}_{g,T_0+T_1}|}\textstyle\sum_{i\in\mathcal{G}_{g,T_0+T_1}} R_i$, \quad
        $\sigma_g^2 \gets \frac{1}{|\mathcal{G}_{g,T_0+T_1}|}\textstyle\sum_{i\in\mathcal{G}_{g,T_0+T_1}}(R_i - \mu_g)^2$
    \State Compute group-normalized advantage: $A_i \gets (R_i - \mu_g) / \sqrt{\sigma_g^2 + \epsilon}$ for all $i\in\mathcal{G}_{g,T_0+T_1}$

    \Statex \hspace{\algorithmicindent}\textcolor{gray}{\textit{\% JS-based token credit allocation (high-temperature group only)}}
    \For{each high-temperature rollout $i \in \mathcal{G}_{g,T_1}$}
        \State Obtain logits $z_{i,t}$ at each response position $t \in \mathcal{T}_i$
        \State $p_{i,t}^{(0)} \gets \mathrm{softmax}(z_{i,t}/T_0)$, \quad $p_{i,t}^{(1)} \gets \mathrm{softmax}(z_{i,t}/T_1)$
        \State $J_{i,t} \gets \mathrm{JS}(p_{i,t}^{(0)}, p_{i,t}^{(1)})$ for all $t\in\mathcal{T}_i$
\State $\bar{J}_i \gets \frac{1}{N_i}\sum_{\tau \in \mathcal{T}_i} J_{i,\tau}$ \Comment{Within-trajectory mean}
\State $\omega_{i,t} \gets \log\!\bigl(1 + (J_{i,t}+\epsilon)/(\bar{J}_i+\epsilon)\bigr)$ \Comment{Relative JS with log compression}
\State $\bar{\omega}_i \gets \frac{1}{N_i}\sum_{\tau \in \mathcal{T}_i} \omega_{i,\tau}$ \Comment{Within-trajectory mean score}
\State $w_{i,t} \gets \omega_{i,t} / \bar{\omega}_i$ \Comment{Unit-mean token weights}
\State $\widetilde{A}_{i,t} \gets w_{i,t} \cdot A_i$ \Comment{Token-level TGRL advantage}
    \EndFor

    \Statex \hspace{\algorithmicindent}\textcolor{gray}{\textit{\% Low-temperature group: contributes to $\mu_g, \sigma_g$ but receives no gradient}}
    \State $\widetilde{A}_{i,t} \gets 0$ for all $i \in \mathcal{G}_{g,T_0}$, $t \in \mathcal{T}_i$

\EndFor

\Statex \textcolor{gray}{\textit{\% Policy update}}
\State Compute importance ratios: $r_{i,t}(\theta) \gets \pi_\theta^{T_1}(y_{i,t}\mid x_g, y_{i,<t})\, /\, \pi_{\theta_{\mathrm{old}}}^{T_1}(y_{i,t}\mid x_g, y_{i,<t})$
\State $\mathcal{L}_{\mathrm{TGRL}}(\theta) \gets -\mathbb{E}_g\!\Big[\frac{1}{|\mathcal{G}_{g,T_1}|}\sum_{i\in\mathcal{G}_{g,T_1}}\frac{1}{N_i}\sum_{t\in\mathcal{T}_i} \min\!\big(r_{i,t}\widetilde{A}_{i,t},\; \mathrm{clip}(r_{i,t}, 1{-}\epsilon_c, 1{+}\epsilon_c)\,\widetilde{A}_{i,t}\big)\Big]$
\State Update $\theta$ by minimizing $\mathcal{L}_{\mathrm{TGRL}}(\theta)$
\end{algorithmic}
\end{algorithm}
\newpage

\section{Time Complexity Analysis}
\label{app:time_complexity}

We analyze the per-update time complexity of TGRL in Algorithm~\ref{alg:tgrl} and compare it with standard single-temperature group-based RLVR. The goal is to isolate the algorithm-specific overhead introduced by mixed-temperature grouping and JS-based token credit allocation. Let $G$ denote the number of prompt groups in one policy-update step. For prompt $x_g$, let $P_g$ be the prompt length, and for each rollout $i\in\mathcal{G}_{g,T_0+T_1}$, let $N_i = |\mathcal{T}_i|$ denote the number of valid response tokens. Define
\begin{equation}
\label{eq:complexity_token_sums}
    S_g \coloneqq \sum_{i\in\mathcal{G}_{g,T_0+T_1}} N_i,
    \qquad
    S_g^{T_1} \coloneqq \sum_{i\in\mathcal{G}_{g,T_1}} N_i.
\end{equation}
We further let $U\ge 1$ denote the number of optimizer epochs applied to the same rollout batch. Finally, let $C_{\mathrm{ver}}(g)$ denote the verifier cost on prompt group $g$. This term is task-dependent, but it is shared by TGRL and the single-temperature baselines, so it appears only as an additive term in the comparison below.

\paragraph{Computation model.}
Consider a decoder-only Transformer with $L$ layers, hidden width $d$, and vocabulary size $|V|$. Under standard KV-cached decoding, generating one new token at context length $s$ requires
\begin{equation}
\label{eq:complexity_decode_step}
    c_{\mathrm{dec}}(s)
    =
    \Theta\!\bigl(Ld^2 + Lsd + d|V|\bigr),
\end{equation}
where the three terms correspond to the dense layer projections/MLPs, causal attention against the cached prefix, and the language-model head. Therefore, generating a response of length $N_i$ conditioned on prompt length $P_g$ costs
\begin{align}
\label{eq:complexity_generation_single}
    C_{\mathrm{gen}}(P_g,N_i)
    &=
    \sum_{t=1}^{N_i} c_{\mathrm{dec}}(P_g+t-1) \nonumber\\
    &=
    \Theta\!\left(
        L N_i d^2
        + L d\sum_{t=1}^{N_i}(P_g+t-1)
        + N_i d|V|
    \right) \nonumber\\
    &=
    \Theta\!\left(
        L N_i d^2
        + L N_i P_g d
        + \frac{L}{2}N_i(N_i-1)d
        + N_i d|V|
    \right).
\end{align}

For the actor update, TGRL performs teacher-forced evaluation on the sampled trajectories to obtain token logits and importance ratios. A forward-backward pass over a sequence of total length $s$ has the standard Transformer complexity
\begin{equation}
\label{eq:complexity_teacher_forced}
    c_{\mathrm{tf}}(s)
    =
    \Theta\!\bigl(Lsd^2 + Ls^2d + sd|V|\bigr),
\end{equation}
so the model-update cost for one rollout $i$ is
\begin{equation}
\label{eq:complexity_update_single}
    C_{\mathrm{upd}}(P_g,N_i)
    =
    \Theta\!\bigl(c_{\mathrm{tf}}(P_g+N_i)\bigr)
    =
    \Theta\!\bigl(L(P_g+N_i)d^2 + L(P_g+N_i)^2 d + (P_g+N_i)d|V|\bigr).
\end{equation}

\begin{proposition}[Per-update time complexity of TGRL]
\label{prop:tgrl_time_complexity}
Assume that the current-policy logits computed on the high-temperature trajectories are reused both for the clipped-ratio term in Eq.~\eqref{eq:tgrl_ratio} and for the JS scores in Eq.~\eqref{eq:tgrl_js_def}. Then one TGRL policy-update step over $G$ prompt groups has time complexity
\begin{equation}
\label{eq:complexity_tgrl_total}
\begin{aligned}
\mathcal{T}_{\mathrm{TGRL}}
= \sum_{g=1}^{G} \Biggl[
    \sum_{i\in\mathcal{G}_{g,T_0+T_1}} C_{\mathrm{gen}}(P_g,N_i)
    + U\sum_{i\in\mathcal{G}_{g,T_1}} C_{\mathrm{upd}}(P_g,N_i) \\
    + O\!\left(
        S_g^{T_1} |V|
        + |\mathcal{G}_{g,T_0+T_1}|
    \right)
    + C_{\mathrm{ver}}(g)
\Biggr].
\end{aligned}
\end{equation}
\end{proposition}

\begin{proof}
Algorithm~\ref{alg:tgrl} decomposes into four parts.

\emph{(i) Mixed-temperature rollout.} For each prompt $x_g$, TGRL samples trajectories indexed by $\mathcal{G}_{g,T_0}$ and $\mathcal{G}_{g,T_1}$. The total rollout cost is therefore $\sum_{i\in\mathcal{G}_{g,T_0+T_1}} C_{\mathrm{gen}}(P_g,N_i)$.

\emph{(ii) Grouped reward statistics.} Once rewards are returned by the verifier, computing $\mu_g$, $\sigma_g^2$, and the rollout-level advantages $A_i$ requires only linear work in the group size, namely $O(|\mathcal{G}_{g,T_0+T_1}|)$.

\emph{(iii) High-temperature-group model pass.} By design, only the high-temperature group receives non-zero token advantages in Eq.~\eqref{eq:tgrl_token_adv}. Hence only trajectories in $\mathcal{G}_{g,T_1}$ participate in the actor update, giving the term $U\sum_{i\in\mathcal{G}_{g,T_1}} C_{\mathrm{upd}}(P_g,N_i)$.

\emph{(iv) JS-based credit allocation and weight normalization.} For each high-temperature rollout and each valid token position, computing $p_{i,t}^{(0)}$, $p_{i,t}^{(1)}$, and $J_{i,t}$ from a shared logit vector $z_{i,t}$ requires vocabulary-wise softmax and elementwise operations, which cost $O(|V|)$ per token. Summed over the high-temperature group, this yields $O(S_g^{T_1}|V|)$. The subsequent weight map in Eq.~\eqref{eq:tgrl_js_mean} consists of two within-trajectory means $(\bar{J}_i,\bar{\omega}_i)$, a pointwise $\log(1+\cdot)$, and a final pointwise division, all linear in $N_i$ and therefore absorbed into the $O(S_g^{T_1}|V|)$ term. Adding all terms and then summing over $g=1,\dots,G$ proves Eq.~\eqref{eq:complexity_tgrl_total}.
\end{proof}

\begin{corollary}[Comparison with single-temperature group optimization]
\label{cor:tgrl_baseline_complexity}
Consider a standard single-temperature group-based RLVR method that uses the same rollout budget $|\mathcal{G}_{g,T_0+T_1}|$ and applies the actor update to all trajectories in $\mathcal{G}_{g,T_0+T_1}$. Its per-update complexity is
\begin{equation}
\label{eq:complexity_baseline_total}
    \mathcal{T}_{\mathrm{base}}
    =
    \sum_{g=1}^{G}
    \left[
        \sum_{i\in\mathcal{G}_{g,T_0+T_1}} C_{\mathrm{gen}}(P_g,N_i)
        +
        U\sum_{i\in\mathcal{G}_{g,T_0+T_1}} C_{\mathrm{upd}}(P_g,N_i)
        + O(|\mathcal{G}_{g,T_0+T_1}|)
        + C_{\mathrm{ver}}(g)
    \right].
\end{equation}
Consequently,
\begin{equation}
\label{eq:complexity_difference_tgrl_base}
    \mathcal{T}_{\mathrm{TGRL}} - \mathcal{T}_{\mathrm{base}}
    =
    - U\sum_{g=1}^{G}\sum_{i\in\mathcal{G}_{g,T_0}} C_{\mathrm{upd}}(P_g,N_i)
    +
    O\!\left(
        \sum_{g=1}^{G}
        S_g^{T_1} |V|
    \right).
\end{equation}
In particular, when $|\mathcal{G}_{g,T_0}|$, $|\mathcal{G}_{g,T_1}|$, $|\mathcal{G}_{g,T_0+T_1}|$, and $U$ are fixed hyperparameters independent of model size and sequence length, TGRL has the same leading asymptotic order as standard group-based RLVR; its additional cost is a lower-order token-wise JS post-processing term on the high-temperature group.
\end{corollary}

\paragraph{Interpretation.}
Corollary~\ref{cor:tgrl_baseline_complexity} shows that TGRL does not alter the dominant scaling law of group-based RLVR under a fixed rollout budget. The reason is structural: the low-temperature group is used for detection through the mixed-group statistics, but it is excluded from the actor update by Eq.~\eqref{eq:tgrl_token_adv}. As a result, TGRL replaces the model-update cost of the $|\mathcal{G}_{g,T_0}|$ low-temperature trajectories with token-level JS post-processing on the $|\mathcal{G}_{g,T_1}|$ high-temperature trajectories. Since $C_{\mathrm{upd}}(P_g,N_i)$ already contains a full Transformer forward-backward pass and a language-model head cost of order $(P_g+N_i)d|V|$, whereas the extra JS computation is only $O(N_i|V|)$ once logits are available, the added TGRL-specific overhead is strictly lower-order with respect to model width $d$.

A convenient fixed-length upper bound follows immediately. If $P_g\le P_{\max}$ and $N_i\le N_{\max}$ for all groups and trajectories, then
\begin{equation}
\label{eq:complexity_tgrl_fixed_length}
\begin{aligned}
\mathcal{T}_{\mathrm{TGRL}}
= O\Bigl(&
\sum_{g=1}^{G}|\mathcal{G}_{g,T_0+T_1}|\,C_{\mathrm{gen}}(P_{\max},N_{\max})
+ U\sum_{g=1}^{G}|\mathcal{G}_{g,T_1}|\,C_{\mathrm{upd}}(P_{\max},N_{\max}) \\
&+ \sum_{g=1}^{G}|\mathcal{G}_{g,T_1}|\,N_{\max}|V|
+ \sum_{g=1}^{G} C_{\mathrm{ver}}(g)
\Bigr).
\end{aligned}
\end{equation}
Under the training configuration used in this paper ($|\mathcal{G}_{g,T_0+T_1}|=4$, $|\mathcal{G}_{g,T_0}|=1$, $|\mathcal{G}_{g,T_1}|=3$), the dominant actor-update term scales with three rather than four trajectories per prompt group, while the TGRL-specific overhead is confined to the token-wise JS computation on the three high-temperature trajectories.

\paragraph{Naive non-shared implementation.}
The bound above assumes the standard implementation in which the logits used for Eq.~\eqref{eq:tgrl_js_def} are reused when evaluating Eq.~\eqref{eq:tgrl_ratio}. If these logits are computed in a separate forward pass, then Eq.~\eqref{eq:complexity_tgrl_total} acquires one additional high-temperature-group forward term of order $\sum_{g=1}^{G}\sum_{i\in\mathcal{G}_{g,T_1}} \Theta\!\bigl(c_{\mathrm{tf}}(P_g+N_i)\bigr)$. This does not change the asymptotic order, but it increases the constant factor.

\clearpage

\section{Wall-Clock Training Efficiency}
\label{app:wallclock}

The analysis in Appendix~\ref{app:time_complexity} established that TGRL shares the leading asymptotic order of single-temperature group-based RLVR, with only a lower-order token-wise JS term on the high-temperature group. This appendix complements that theoretical result with direct wall-clock measurements on the 14B training setup, using the same hardware, data pipeline, and decoding configuration as the main experiments. The primary tracking metric is \textsc{AIME-Combined}, defined as the sum of the Avg@16 validation accuracies on AIME24 and AIME25, logged periodically during training; larger values indicate stronger intermediate checkpoints and the score lies in $[0,2]$. We report two complementary protocols, each designed to rule out a distinct alternative explanation for TGRL's end-of-training accuracy.

\paragraph{Protocols.}
(i)~\textbf{Post-warmup fixed-budget comparison against GRPO.} Because TGRL activates mixed-temperature grouping only after a $20$-step single-temperature warmup (Section~\ref{sec:algorithm}), comparing wall-clock from step $0$ against a baseline without such a warmup would implicitly charge TGRL for time spent in a regime where both methods are optimizing an identical single-temperature objective. To isolate the two training regimes themselves, we align the clock to the end of TGRL's $20$-step warmup and compare the best \textsc{AIME-Combined} attained within the subsequent $12$ hours. Absolute wall-clock times (i.e., measured from step $0$ on TGRL's side) are also listed for transparency.

(ii)~\textbf{From-start comparison against DAPO.} DAPO, the strongest single-temperature baseline in our math suite, does not employ a method-specific warmup schedule beyond standard learning-rate warmup. We therefore compare both methods from step $0$ within the same $12$-hour window, and additionally record the first time each method reaches a common target accuracy, chosen as $0.95\times$ DAPO's $12$-hour best $=0.9045$. The from-start time-to-target protocol is the most direct measure of learning efficiency because it is independent of any specific stopping criterion.

\begin{table}[h]
\centering
\caption{Wall-clock training efficiency on the 14B setup under identical hardware, data, and evaluation settings. ``Time of Best (h)'' is measured from the protocol origin: the end of TGRL's $20$-step warmup in (i), and step $0$ in (ii)--(iii). ``Best AIME-Comb.'' is the highest validation score reached within the budget; ``Step'' is the step index at which that best was attained.}
\label{tab:wallclock}
\small
\setlength{\tabcolsep}{5pt}
\begin{tabular}{llcccl}
\toprule
\textbf{Protocol} & \textbf{Method} & \textbf{Best AIME-Comb.} & \textbf{Time of Best (h)} & \textbf{Step} & \textbf{Remarks} \\
\midrule
\multicolumn{6}{c}{\textit{(i) Post-warmup $12$-hour fixed-budget comparison (TGRL vs.\ GRPO)}} \\
\midrule
TGRL vs.\ GRPO & \textbf{TGRL (ours)} & \textbf{1.0833} & 11.01 & 80  & $13.84$h absolute \\
TGRL vs.\ GRPO & GRPO                & 1.0354          & 8.86  & 100 & $11.42$h absolute \\
\midrule
\multicolumn{6}{c}{\textit{(ii) From-start $12$-hour fixed-budget comparison (TGRL vs.\ DAPO)}} \\
\midrule
TGRL vs.\ DAPO & \textbf{TGRL (ours)} & \textbf{1.0604} & 10.21 & 60 & from training start \\
TGRL vs.\ DAPO & DAPO                & 0.9521          & 10.01 & 30 & from training start \\
\midrule
\multicolumn{6}{c}{\textit{(iii) From-start time-to-target comparison (target $=0.95\times$ DAPO's $12$h-best $=0.9045$)}} \\
\midrule
TGRL vs.\ DAPO & \textbf{TGRL (ours)} & 0.9250 & \textbf{6.39} & 40 & first time reaching target \\
TGRL vs.\ DAPO & DAPO                & 0.9521 & 10.01         & 30 & first time reaching target \\
\bottomrule
\end{tabular}
\end{table}

\paragraph{Analysis.}
Three observations emerge from Table~\ref{tab:wallclock}.

\emph{(1) TGRL attains a higher post-warmup best score than GRPO under the same wall-clock budget.} Within the shared $12$-hour post-warmup budget, TGRL reaches a best \textsc{AIME-Combined} of $1.0833$ versus GRPO's $1.0354$, a $+0.048$ absolute improvement in validation score. Crucially, this gain is obtained despite TGRL performing mixed-temperature rollout construction together with per-token JS computation on the high-temperature group at every step. The extra per-step work is small enough relative to the dominant Transformer forward-backward cost that TGRL still delivers \emph{more} validation progress per unit wall-clock, consistent with Corollary~\ref{cor:tgrl_baseline_complexity}.

\emph{(2) TGRL remains stronger than DAPO when wall-clock is measured from the start of training.} In the stricter from-start protocol, TGRL attains a best \textsc{AIME-Combined} of $1.0604$ in the first $12$ hours while DAPO attains $0.9521$, a $+0.108$ absolute gap---notably wider than in (i), because TGRL's post-warmup advantage compounds with the shared early-training period. Under this accounting, TGRL's first $20$ single-temperature warmup steps are included in the wall-clock total, and the method still outpaces a strong RLVR baseline on identical hardware and data.

\emph{(3) TGRL reaches DAPO-level accuracy substantially earlier.} Under the time-to-target protocol, TGRL crosses the $0.9045$ threshold at $6.39$ hours of wall-clock, whereas DAPO reaches its own best of $0.9521$ only after $10.01$ hours. TGRL therefore saves roughly $3.6$ hours---approximately $36\%$ of the wall-clock budget---to attain an accuracy level that DAPO achieves only near the end of the measurement window. Because this protocol fixes the accuracy target and measures time rather than fixing time and measuring accuracy, it is the most direct single-number summary of TGRL's learning-efficiency advantage.

\paragraph{Interpretation of the wall-clock measurements.}
The wall-clock results support three points.
First, they match the asymptotic conclusion of Corollary~\ref{cor:tgrl_baseline_complexity} in a practical regime: the mixed-temperature rollout construction and token-wise JS post-processing on the high-temperature group do not translate into a visible wall-clock penalty against single-temperature baselines under a fixed rollout budget $|\mathcal{G}_{g,T_0+T_1}|{=}4$.
Second, they show that the accuracy advantage reported in Tables~\ref{tab:math1} and~\ref{code_1} is accompanied by stronger learning progress per unit wall-clock under matched hardware and rollout budgets: TGRL is simultaneously faster-converging and achieves a higher peak than both GRPO (under the comparable post-warmup protocol) and DAPO (under the strict from-start protocol).
Third, they are consistent with the mechanism analysis of Section~\ref{sec:mechanism_analysis}: by turning exploration into a measurable prompt-level gain and allocating it as token-level credit to temperature-sensitive tokens, TGRL translates each compute unit into more learning progress while retaining stable optimization behavior.

\section{Exploration-Oriented Baselines}
\label{app:entropy_baselines}

As a mechanism-level complement to the single-temperature RLVR baselines in Table~\ref{tab:math1}, we compare TGRL against two families of exploration-oriented baselines on the Qwen3-14B mathematics setup: (i)~entropy-regularized GRPO with entropy coefficient $\lambda\in\{10^{-4},10^{-3}\}$, and (ii)~TAMPO, a recent temperature-adaptive RLVR method that adapts the decoding temperature globally across training steps. As discussed in Section~\ref{related_work}, these two families cover the dominant exploration-encouragement strategies in RLVR---loss-level entropy bonuses and rollout-level temperature scheduling---so together they probe whether TGRL's gains can be recovered by alternatives to within-prompt temperature grouping. All other training and evaluation settings are identical to the 14B configuration in Section~\ref{sec:experiments}, including the $|\mathcal{G}_{g,T_0+T_1}|{=}4$ rollout budget, the $20$-step warmup schedule, and the Avg@16 evaluation protocol on the six mathematical reasoning benchmarks.
Table~\ref{tab:entropy_baselines_14b} extends the main comparison in Table~\ref{tab:math1} with these exploration-oriented baselines. Three observations are important.

\paragraph{Both baseline families improve over plain GRPO, but unevenly.}
A mild entropy bonus ($\lambda{=}10^{-4}$) raises the six-benchmark average from $67.0$ to $67.8$, and TAMPO reaches $68.0$---numerically matching DAPO, the strongest single-temperature baseline in Table~\ref{tab:math1}. Both mechanisms therefore deliver measurable exploration gains, consistent with prior reports that entropy regularization~\citep{haarnoja2018soft,cui2025entropy} and temperature scheduling~\citep{yang2025let,dang2026temperature} are meaningful RLVR design choices. However, their benchmark-level profiles are uneven. Entropy regularization concentrates its gains on AIME24 ($+5.7$) and AIME25 ($+3.0$) but degrades AMC23 ($-2.4$), MATH500 ($-0.7$), and Minerva ($-1.8$). TAMPO shows a large improvement on AIME24 ($+5.8$), smaller gains on AIME25 ($+1.7$) and MATH500 ($+0.8$), and mild regressions on AMC23 ($-0.9$), Minerva ($-0.5$), and Olympiad ($-1.2$). Neither baseline improves consistently across all six benchmarks, and raising the entropy coefficient to $\lambda{=}10^{-3}$ further illustrates the hyperparameter sensitivity of loss-level regularization, dropping the average back to $66.7$.

\paragraph{The limitation is structural: both interventions are applied prompt-agnostically.}
Entropy regularization adds a constant entropy term to the actor loss, uniformly encouraging stochastic policies on every prompt regardless of whether broader search helps on that prompt. TAMPO changes the decoding temperature as a function of training progress, but the rollout group on each prompt is still generated at a \emph{single} temperature and normalized under standard single-temperature group statistics; the optimizer therefore cannot tell whether the current global temperature is helping on a given prompt. Both interventions operate at the same resolution: global and prompt-agnostic. On prompts where broader search would help, a global exploration push is beneficial; on prompts that are already well-supported by the baseline decoding regime, the same push is neutral or harmful. The mixed per-benchmark pattern---large gains on competition-style benchmarks, flat or negative changes on benchmarks where low-temperature decoding is already near-optimal---is exactly what this prompt-agnostic design predicts.

\paragraph{TGRL's prompt-level grouping targets this limitation directly.}
TGRL forms low- and high-temperature subgroups \emph{within every prompt group}, so the reward gap $\Delta R_g$ yields an unbiased per-prompt estimate of exploration gain (Lemma~\ref{lem:dual_var}). Proposition~\ref{prop:mixed_adv} then shifts every high-temperature advantage by $(|\mathcal{G}_{g,T_0}|/|\mathcal{G}_{g,T_0+T_1}|)\,\Delta R_g/\sqrt{\sigma_g^2+\epsilon}$, turning the prompt-specific estimate into a prompt-specific update. With $69.4$ average accuracy, TGRL outperforms all four exploration-oriented baselines: $+1.6$ over the stronger entropy-regularized run, $+1.4$ over TAMPO, and $+2.4$ over plain GRPO. The margins are largest on the benchmarks where prompt-level exploration gain is most informative: Olympiad ($+2.6$ over the best non-TGRL entry) and AIME25 ($+1.0$). On benchmarks where most prompts are already well-supported by low-temperature decoding (AMC23, Minerva), TGRL is competitive with the best non-TGRL baseline rather than dominating it---but it also avoids the regressions that global interventions incur on these benchmarks. The overall pattern is consistent with the theoretical role of the mixed-temperature group as a prompt-specific exploration-gain estimator rather than a global exploration regularizer.

\paragraph{Supplementary results on Qwen3-4B.}
\label{app:math_4b_supplementary}

For completeness, we additionally report the 4B counterpart of Table~\ref{tab:entropy_baselines_14b} on the same six mathematical reasoning benchmarks, using Qwen3-4B and otherwise identical training and evaluation settings. The intent is to verify that the main-paper ordering transfers to a smaller-scale backbone.

\begin{table}[h]
    \centering
    \caption{Avg@16 results on mathematical reasoning benchmarks with Qwen3-4B trained in \texttt{think} mode. \textbf{Best} results are in \textbf{bold}, second-best results are \underline{underlined}.}
    \label{tab:math_4b_methods}
    \resizebox{\textwidth}{!}{
    \begin{tabular}{l|cccccc|c}
        \toprule
        \textbf{Method} & \textbf{AIME24} & \textbf{AIME25} & \textbf{AMC23} & \textbf{MATH500} & \textbf{Minerva} & \textbf{Olympiad} & \textbf{Average} \\
        \midrule
        Qwen3-4B & 49.4 & 35.6 & 78.8 & 84.6 & 40.3 & 58.2 & 57.8 \\
        GRPO & \underline{54.2} & 39.2 & 83.9 & 90.8 & 44.4 & 64.2 & 62.8 \\
        Dr.GRPO & 49.0 & \underline{40.0} & 85.6 & \textbf{92.7} & \textbf{45.1} & \underline{66.4} & 63.1 \\
        RLOO & 53.1 & 39.0 & \underline{85.8} & 81.8 & 40.4 & 57.6 & 59.6 \\
        TAMPO & 54.0 & 39.4 & \textbf{85.9} & 92.2 & 44.2 & 64.2 & \underline{63.3} \\
        \textbf{TGRL (ours)} & \textbf{56.2} & \textbf{42.7} & 81.4 & \underline{92.5} & \underline{44.9} & \textbf{69.0} & \textbf{64.5} \\
        \bottomrule
    \end{tabular}}
    \vspace{-0.5em}
\end{table}

Table~\ref{tab:math_4b_methods} reproduces the 14B ordering at the 4B scale: TGRL achieves the best six-benchmark average ($64.5$), ahead of TAMPO ($63.3$, $+1.2$) and Dr.GRPO ($63.1$, $+1.4$), with the largest per-benchmark improvements over the best non-TGRL baseline on AIME25 ($+2.7$) and Olympiad ($+2.6$). We include this table as supplementary evidence that TGRL's benefits do not hinge on the capacity of 14B/32B backbones.

\clearpage
\section{Sensitivity to the Fixed Exploration Temperature on 4B}
\label{app:temp_sweep_4b}

To isolate the role of the exploration temperature from the rest of the method, we run TGRL on the 4B model with a fixed low-temperature reference $T_0=0.3$ and sweep the exploration temperature over
\[
T_1 \in \{0.8,\,1.0,\,1.2,\,1.4,\,1.6\}.
\]
All runs are analyzed over the first 150 training steps. This study serves two purposes. First, it maps how performance varies with the absolute value of the exploration temperature. Second, it explains why we use $T_1=1.2$ as the default fixed exploration temperature in the main experiments.

\begin{table}[h]
    \centering
    \caption{
    Summary of the 4B fixed-temperature sweep over the first 150 training steps. $\Delta T$ denotes the gap relative to the fixed low-temperature reference $T_0=0.3$. AUC is computed from the validation AIME-combined curve. The results show that $T_1=1.2$ offers the strongest early learning efficiency, whereas larger temperatures remain competitive but less sample-efficient.
    }
    \label{tab:temp_sweep_4b_summary}
    \resizebox{0.95\linewidth}{!}{
    \begin{tabular}{c|cccccc}
        \toprule
        \textbf{$T_1$} & \textbf{$\Delta T$} & \textbf{Peak} & \textbf{Peak step} & \textbf{Final@150} & \textbf{AUC@150} & \textbf{Step to 0.94} \\
        \midrule
        0.8 & 0.5 & 0.808 & 130 & 0.758 & 0.698 & -- \\
        1.0 & 0.7 & 0.942 & 140 & 0.900 & 0.822 & 140 \\
        1.2 & 0.9 & 0.963 & 70  & 0.933 & \textbf{0.885} & \textbf{70} \\
        1.4 & 1.1 & 0.923 & 50  & 0.900 & 0.846 & -- \\
        1.6 & 1.3 & \textbf{0.975} & 140 & \textbf{0.948} & 0.814 & 110 \\
        \bottomrule
    \end{tabular}}
\end{table}

Table~\ref{tab:temp_sweep_4b_summary} summarizes the aggregate AIME-combined trend numerically. The performance collapse appears only when the temperature gap is too small. Once the gap is sufficiently pronounced, the method is fairly robust to the precise absolute value of $T_1$. Within that broad robust regime, $T_1=1.2$ provides the best early-to-mid training efficiency: it is the fastest setting to cross the strong-performance threshold and yields the highest validation AUC up to step 150, making it a sensible default within a broader competitive regime.

\begin{figure}[h]
    \centering
    \includegraphics[width=\linewidth]{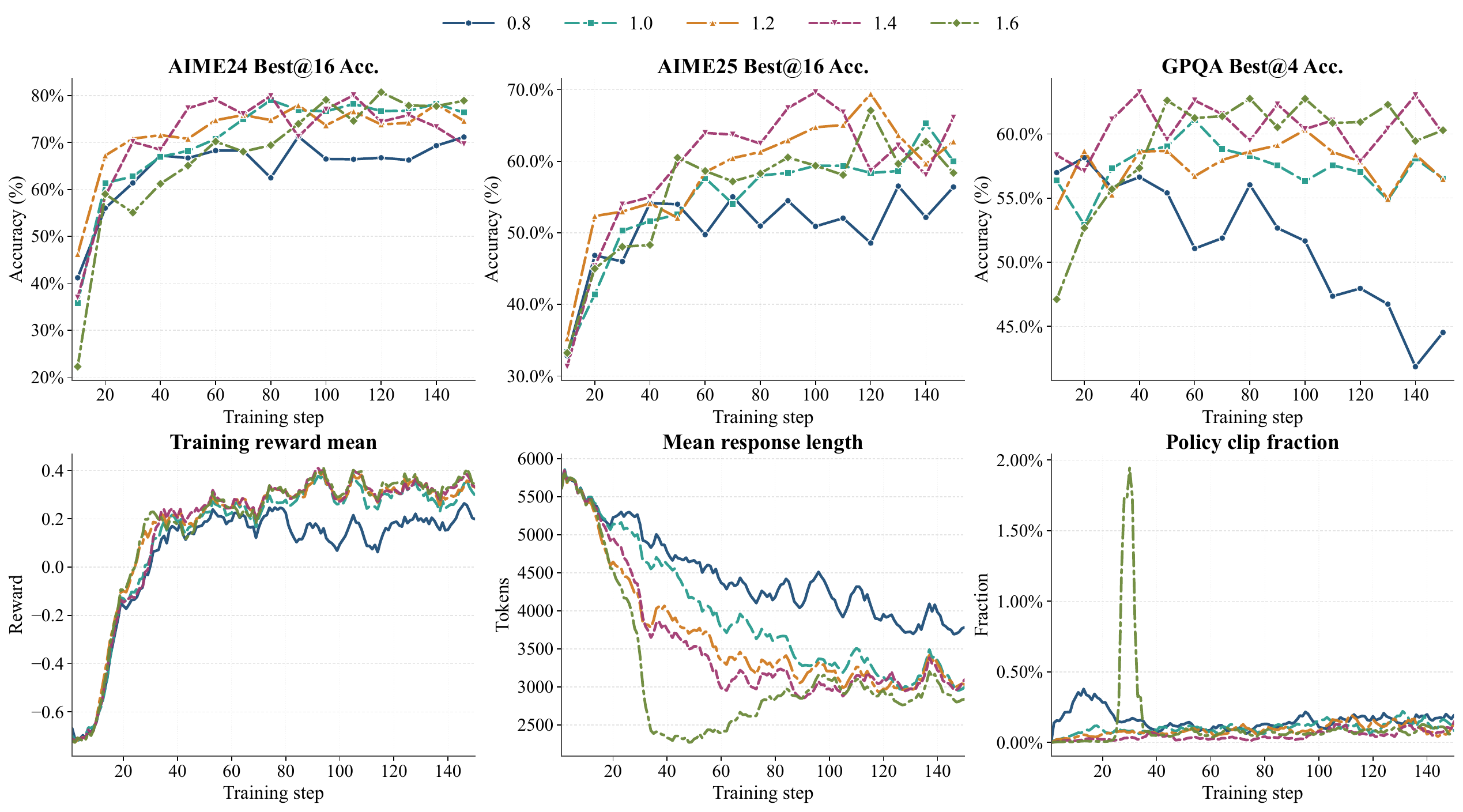}
    \caption{
    Per-benchmark validation and optimization diagnostics for the 4B fixed-temperature sweep. Top: AIME24 Best@16, AIME25 Best@16, and GPQA Best@4 diagnostic. Bottom: training reward, response length, and policy clip fraction. Moderate-to-large temperature gaps define a broad robust operating region, within which $T_1=1.2$ offers a particularly attractive efficiency--stability trade-off.
    }
    \label{fig:temp_sweep_4b_diagnostics}
\end{figure}

Figure~\ref{fig:temp_sweep_4b_diagnostics} provides the per-benchmark and optimization-level view behind that aggregate summary. On the benchmark side, the smallest gap $T_1=0.8$ is consistently weak across AIME24, AIME25, and GPQA, indicating that insufficient separation from the low-temperature group yields a poor grouped exploration signal. Once the gap is increased, performance enters a clearly more robust but still non-monotonic regime, and all settings from $T_1=1.0$ to $1.6$ remain competitive for substantial portions of training. Within this broad robust region, $T_1=1.2$ offers the cleanest overall balance: it reaches a strong regime earliest on AIME24, stays near the top throughout AIME25, and remains fully competitive on GPQA. At the same time, the neighboring higher-temperature settings also remain competitive---$T_1=1.4$ is particularly strong on AIME25 and GPQA, while $T_1=1.6$ attains some of the highest late AIME24 and GPQA values. The per-benchmark decomposition therefore sharpens the practical conclusion of the sweep: TGRL is robust to a fairly wide range of exploration temperatures, and $T_1=1.2$ is selected as a consistently strong default across tasks and checkpoints.

The optimization panels explain why $T_1=1.2$ is preferred as the default. As $T_1$ increases, late-stage training reward rises and response length generally falls, but validation quality does not improve monotonically with either proxy. At the same time, the higher-temperature settings remain strong: $T_1=1.4$ and $T_1=1.6$ both achieve strong late rewards and short responses, consistent with the benchmark-side evidence that TGRL remains effective over a broad temperature range. The main distinction is that $T_1=1.6$ shows the clearest transient spike in policy clip fraction early in training, whereas $T_1=1.2$ already reaches near-best reward, operates in the short-response regime, and maintains a uniformly low clip fraction with smoother update dynamics. Taken together with Table~\ref{tab:temp_sweep_4b_summary}, these diagnostics support selecting the default exploration temperature by joint efficiency, cross-benchmark robustness, and optimization stability; under this criterion, $T_1=1.2$ is a particularly balanced and competitive operating point within a much broader robust region.

\section{Multi-Seed Evaluation on Competition Benchmarks}
\label{app:multi_seed}

Competition-level benchmarks are relatively small in scale: AIME24 contains 30 problems, AIME25 contains 30 problems, and AMC23 contains 40 problems in the commonly used open-source evaluation setup. As a result, single-run Avg@16 estimates can exhibit non-trivial sampling variance. To quantify this variance and verify the robustness of the gains reported in Table~\ref{tab:math1}, we repeat the Avg@16 evaluation with five random seeds for TGRL and GRPO at both the 14B and 32B scales on all three benchmarks. Table~\ref{tab:multi_seed} reports the mean $\pm$ standard deviation.

\begin{table}[h]
    \centering
    \caption{Multi-seed Avg@16 evaluation on competition-level math benchmarks. We report mean accuracy (\%) $\pm$ standard deviation across 5 random seeds.}
    \label{tab:multi_seed}
    \begin{tabular}{l|ccc}
        \toprule
        \textbf{Method} & \textbf{AIME24} & \textbf{AIME25} & \textbf{AMC23} \\
        \midrule
        \multicolumn{4}{c}{\textbf{Qwen3-14B}} \\
        \midrule
        GRPO & $57.3 \pm 3.8$ & $43.2 \pm 1.8$ & $92.4 \pm 1.2$ \\
        \textbf{TGRL (ours)} & $64.2 \pm 4.1$ & $48.7 \pm 3.5$ & $91.9 \pm 0.4$ \\
        \midrule
        \multicolumn{4}{c}{\textbf{Qwen3-32B}} \\
        \midrule
        GRPO & $59.5 \pm 2.5$ & $45.3 \pm 2.6$ & $87.9 \pm 1.4$ \\
        \textbf{TGRL (ours)} & $65.3 \pm 3.5$ & $48.3 \pm 3.3$ & $91.8 \pm 1.2$ \\
        \bottomrule
    \end{tabular}
\end{table}

Two observations emerge from the multi-seed results. First, the cross-seed standard deviations are uniformly small (typically 1--4 percentage points), showing that the Avg@16 metric yields stable evaluations and that the rankings in Table~\ref{tab:math1} are stable across seeds. Second, TGRL's mean accuracy is consistently higher than GRPO's across the majority of benchmark--scale combinations, with the clearest separation on the challenging AIME benchmarks: TGRL leads by 3.0--6.9 points across scales, and the mean gaps consistently exceed the individual standard deviations. These results provide statistical evidence that the main-table improvements are reproducible and robust to evaluation variance.

\clearpage
\section{Robustness under Larger Rollout Budget}
\label{app:n8_robustness}

The main experiments use a rollout budget of $|\mathcal{G}_{g,T_0+T_1}|{=}4$ per prompt. To assess TGRL's robustness to the rollout budget setting, we re-run TGRL and GRPO on Qwen3-14B mathematics with a doubled rollout budget of $n{=}8$. TGRL allocates two responses to the low-temperature reference group ($|\mathcal{G}_{g,T_0}|{=}2$) and six to the high-temperature exploration group ($|\mathcal{G}_{g,T_1}|{=}6$), preserving the 1:3 reference-to-exploration ratio used in the main experiments. GRPO uses all eight responses at a single temperature. Table~\ref{tab:n8_robustness} reports Avg@16 accuracy on the six mathematical reasoning benchmarks.

\begin{table}[h]
    \centering
    \caption{Avg@16 results on mathematical reasoning benchmarks with Qwen3-14B under a rollout budget of $n{=}8$ per prompt. TGRL allocates $(|\mathcal{G}_{g,T_0}|, |\mathcal{G}_{g,T_1}|){=}(2,6)$, preserving the 1:3 reference-to-exploration ratio. \textbf{Best} results are in \textbf{bold}.}
    \label{tab:n8_robustness}
    \resizebox{0.95\textwidth}{!}{
    \begin{tabular}{l|cccccc|c}
        \toprule
        \textbf{Method} & \textbf{AIME24} & \textbf{AIME25} & \textbf{AMC23} & \textbf{MATH500} & \textbf{Minerva} & \textbf{Olympiad} & \textbf{Average} \\
        \midrule
        GRPO ($n{=}8$)        & 60.8 & 46.7 & \textbf{92.5} & \textbf{94.4} & 48.7 & 64.1 & 67.9 \\
        \textbf{TGRL ($n{=}8$, ours)} & \textbf{66.0} & \textbf{49.4} & 92.2 & 93.7 & \textbf{48.9} & \textbf{64.3} & \textbf{69.1} \\
        \bottomrule
    \end{tabular}}
\end{table}

Three observations support the robustness of TGRL's design under the larger budget. First, TGRL achieves a higher six-benchmark average ($69.1\%$ vs.\ $67.9\%$ for GRPO, a $+1.2$ percentage-point gain), consistent with the $+1.4$-point advantage observed at $n{=}4$ on the same model scale (Table~\ref{tab:math1}). Second, the gains are again concentrated on competition-style benchmarks that require diverse reasoning strategies: TGRL improves over GRPO by $5.2\%$ on AIME24 and $2.7$ points on AIME25, where the low-temperature reference baseline provides the clearest signal of whether high-temperature exploration yields a verifiable reward gain. Third, on near-saturated benchmarks (AMC23, MATH500, Minerva, Olympiad), the two methods remain within $0.6\%$ of each other in either direction, which is consistent with TGRL's design: the mixed-temperature grouping applies a prompt-specific exploration push only when the reward gap $\Delta R_g$ is positive, and avoids penalizing prompts where low-temperature decoding is already near-optimal.

Collectively, the $n{=}4$ and $n{=}8$ results show that TGRL's advantage over GRPO is stable across rollout budgets and does not depend on a particular allocation size. The consistent improvement on challenging competition benchmarks across both settings confirms that mixed-temperature grouping reliably identifies and amplifies the value of high-temperature exploration.

\clearpage
\section{Sensitivity to the Low-/High-Temperature Split}
\label{app:split_sensitivity}

We further examine whether the default low-/high-temperature allocation is crucial to TGRL's performance under the same total rollout budget. 
All variants keep the per-prompt rollout budget fixed at $|\mathcal{G}_{g,T_0+T_1}|=4$ and differ only in the split between the low-temperature reference group and the high-temperature exploration group. 
Table~\ref{tab:split_sensitivity} compares the default setting $(|\mathcal{G}_{g,T_0}|,|\mathcal{G}_{g,T_1}|)=(1,3)$ against two alternative allocations, $(2,2)$ and $(3,1)$, on the same six mathematical reasoning benchmarks.

\begin{table}[h]
    \centering
    \caption{Sensitivity of TGRL to the low-/high-temperature rollout split under a fixed total rollout budget of $4$ per prompt. 
    Results are Avg@16 on mathematical reasoning benchmarks. \textbf{Best} results are in \textbf{bold}, second-best results are \underline{underlined}.}
    \label{tab:split_sensitivity}
    \resizebox{0.95\textwidth}{!}{
    \begin{tabular}{l|cccccc|c}
        \toprule
        \textbf{Split} & \textbf{AIME24} & \textbf{AIME25} & \textbf{AMC23} & \textbf{MATH500} & \textbf{Minerva} & \textbf{Olympiad} & \textbf{Average} \\
        \midrule
        $|\mathcal{G}_{g,T_0}|{=}1,\ |\mathcal{G}_{g,T_1}|{=}3$ & \textbf{63.4} & \underline{49.4} & \underline{92.7} & \textbf{95.3} & \textbf{48.9} & \textbf{66.4} & \textbf{69.4} \\
         $|\mathcal{G}_{g,T_0}|{=}2,\ |\mathcal{G}_{g,T_1}|{=}2$ & 60.4 & 48.8 & 87.8 & \underline{93.0} & \underline{48.4} & \underline{64.3} & 67.1 \\
        $|\mathcal{G}_{g,T_0}|{=}3,\ |\mathcal{G}_{g,T_1}|{=}1$ & \underline{61.0} & \textbf{49.8} & \textbf{93.9} & 92.9 & 47.3 & 63.8 & \underline{68.1} \\
        \bottomrule
    \end{tabular}}
\end{table}

Table~\ref{tab:split_sensitivity} shows that TGRL is not overly brittle to the exact split: both alternative allocations remain competitive under the same total rollout budget. 
However, the default $|\mathcal{G}_{g,T_0}|{=}1,\ |\mathcal{G}_{g,T_1}|{=}3$ split gives the best overall average, outperforming $|\mathcal{G}_{g,T_0}|{=}2,\ |\mathcal{G}_{g,T_1}|{=}2$ by $2.3$ points and $|\mathcal{G}_{g,T_0}|{=}3,\ |\mathcal{G}_{g,T_1}|{=}1$ by $1.3$ points.
This pattern supports the design choice used in the main experiments. 
A single low-temperature rollout is sufficient to provide a conservative within-prompt reference for estimating the exploration-gain signal, while allocating three rollouts to the high-temperature branch supplies more exploratory trajectories for both reward-contrast estimation and actor updates. 

The two alternative splits expose the trade-off behind this choice. 
Increasing the baseline allocation to $|\mathcal{G}_{g,T_0}|{=}2,\ |\mathcal{G}_{g,T_1}|{=}2$ gives a more symmetric group but reduces the number of high-temperature trajectories that receive policy-gradient updates, leading to lower average accuracy and a particularly large drop on AMC23. 
Pushing the allocation further to $|\mathcal{G}_{g,T_0}|{=}3,\ |\mathcal{G}_{g,T_1}|{=}1$ strengthens the low-temperature reference and performs well on AIME25 and AMC23, but leaves only one exploratory trajectory for optimization, which weakens the broader cross-benchmark gains and reduces performance on MATH500, Minerva, and Olympiad.
Thus, under a fixed rollout budget, TGRL benefits more from preserving sufficient high-temperature update samples than from using additional low-temperature reference samples. 
The default $|\mathcal{G}_{g,T_0}|{=}1,\ |\mathcal{G}_{g,T_1}|{=}3$ setting therefore provides a favorable balance between a lightweight exploitation reference and a sufficiently large exploration branch.

\clearpage
\section{Synthetic Validation of Mixed-Temperature Grouping}
\label{app:grouping_validation_details}

\begin{figure}[h]
    \centering
    \includegraphics[width=\textwidth]{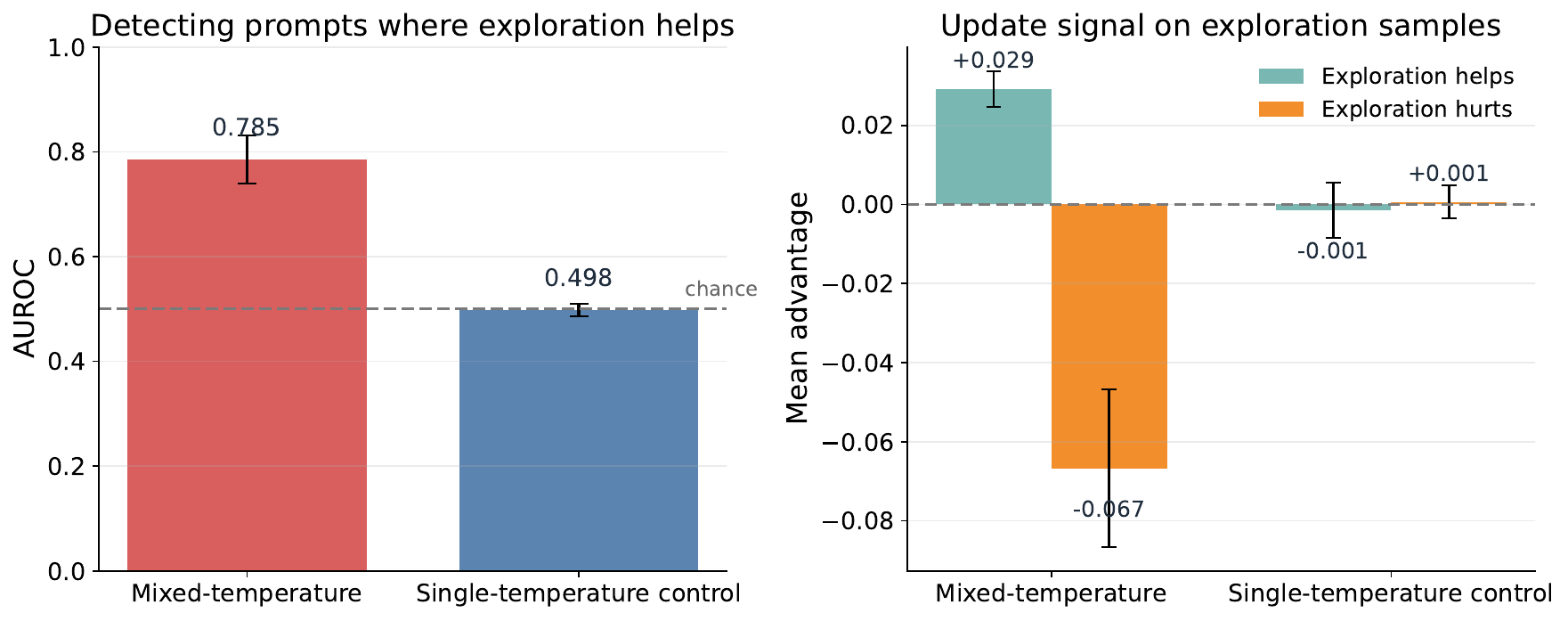}
    \caption{Synthetic validation of mixed-temperature grouping on a tractable two-step autoregressive task. \textbf{Left:} AUROC for detecting prompts on which exploration is beneficial; dashed line: chance. \textbf{Right:} mean advantage assigned to exploration samples. Error bars denote $95\%$ confidence intervals across four random seeds.}
    \label{fig:grouping_gain_carrier}
\end{figure}

This appendix gives the full construction of the synthetic diagnostic referenced in Section~\ref{sec:theory_exploration_gain}. Its purpose is not to provide an additional benchmark, but to test whether mixed-temperature grouping can recover prompt-level exploration gain in a setting where that quantity is exactly measurable.

\paragraph{Task construction.}
The input space consists of all $4$-bit prompts
\[
x=(x_1,x_2,x_3,x_4)\in\{0,1\}^4.
\]
Each prompt requires a two-step autoregressive response: first a binary \emph{branch token} $b\in\{0,1\}$, then a 4-way \emph{answer token} $a\in\{0,1,2,3\}$. The ground-truth branch is
\begin{equation}
\label{eq:toy_branch_target_v2}
b^{*}(x)=x_1\oplus x_2\oplus x_3,
\end{equation}
where $\oplus$ denotes XOR, and the ground-truth answer is
\begin{equation}
\label{eq:toy_answer_target_v2}
a^{*}(x)=2x_3+x_4.
\end{equation}
Reward is binary:
\begin{equation}
\label{eq:toy_reward_v2}
R(x,b,a)=\mathbf{1}\!\left[b=b^{*}(x)\ \wedge\ a=a^{*}(x)\right].
\end{equation}
Thus, a rollout succeeds only when both the early branch decision and the final answer are correct.

\paragraph{Policy architecture.}
The policy is a two-stage MLP with autoregressive factorization
\begin{equation}
\label{eq:toy_factorization_v2}
p_{\theta}(b,a\mid x)=p_{\theta}(b\mid x)\,p_{\theta}(a\mid x,b).
\end{equation}
A two-layer tanh encoder maps the prompt to a hidden state
\begin{equation}
\label{eq:toy_encoder_v2}
h(x)=\tanh\!\bigl(W_2\,\tanh(W_1x+b_1)+b_2\bigr),
\qquad h(x)\in\mathbb{R}^{16}.
\end{equation}
A binary head predicts the branch token:
\begin{equation}
\label{eq:toy_branch_head_v2}
p_{\theta}(b\mid x)=\mathrm{softmax}(U_b h(x)+c_b).
\end{equation}
Given a branch token $b$, we concatenate a learned branch embedding $e(b)\in\mathbb{R}^{4}$ with $h(x)$ and predict the answer through a second MLP head:
\begin{equation}
\label{eq:toy_answer_head_v2}
p_{\theta}(a\mid x,b)
=
\mathrm{softmax}\!\Bigl(
U_a\,\tanh\!\bigl(V[h(x);e(b)]+d\bigr)+c_a
\Bigr).
\end{equation}
This preserves the minimal structure needed for the mechanism test: an early decision token, a later conditioned token, and an autoregressive factorization.

\paragraph{Branch-biased warm start and easy/hard prompts.}
To create a genuine exploration problem without introducing answer-head mismatch, we bias only the branch decision at initialization. Specifically, the warm-start objective is
\begin{equation}
\label{eq:toy_warm_loss_v2}
\mathcal{L}_{\mathrm{warm}}
=
\mathrm{CE}\!\bigl(p_{\theta}(b\mid x),\,0\bigr)
+
\frac{1}{2}\,\mathrm{CE}\!\bigl(p_{\theta}(a\mid x,b{=}0),\,a^{*}(x)\bigr)
+
\frac{1}{2}\,\mathrm{CE}\!\bigl(p_{\theta}(a\mid x,b{=}1),\,a^{*}(x)\bigr).
\end{equation}
The first term biases the policy toward branch $0$, while the last two terms train the answer head under \emph{both} branch conditions. Consequently, hard prompts differ from easy prompts only through the early branch decision, not through missing answer supervision under one branch.

Prompts with $b^{*}(x)=0$ are \emph{easy}, because low-temperature decoding already follows the correct branch. Prompts with $b^{*}(x)=1$ are \emph{hard}, because success requires escaping the initialization bias and selecting the other branch. In our four random seeds, the branch-biased warm start behaves exactly as intended: the mean conditional answer probability under the correct branch at $T_1$ is $0.943$ on both easy and hard prompts, whereas the mean probability of selecting the correct branch on hard prompts is only $0.044$.

\paragraph{Exact prompt-level exploration gain.}
Because both the model and output space are tiny, the exact success probability of every prompt can be computed under any sampling temperature $T$. For a fixed policy,
\begin{equation}
\label{eq:toy_exact_success_v2}
\mu_T(x)
=
p_{\theta}^{(T)}\!\bigl(b^{*}(x)\mid x\bigr)\;
p_{\theta}^{(T)}\!\bigl(a^{*}(x)\mid x,b^{*}(x)\bigr),
\end{equation}
where $p_{\theta}^{(T)}$ denotes the temperature-scaled decoding distribution. We define the exact prompt-level exploration gain by
\begin{equation}
\label{eq:toy_delta_mu_v2}
\Delta \mu(x)\coloneqq \mu_{T_1}(x)-\mu_{T_0}(x).
\end{equation}
Hence $\Delta\mu(x)>0$ means that higher-temperature sampling truly helps on prompt $x$, whereas $\Delta\mu(x)<0$ means that it hurts. Under the branch-biased warm start, positive exploration gain coincides with the hard prompts in all four seeds used in the diagnostic.

\paragraph{Diagnostic protocol.}
All statistics in Figure~\ref{fig:grouping_gain_carrier} are computed at the common warm-start checkpoint, before any RL updates. This ensures that the diagnostic reflects recovery quality and signed advantage structure rather than differences in optimization trajectory. We use
\[
T_0=0.4,\qquad T_1=1.4,
\]
and a total sample budget of $|\mathcal{G}_{x,T_0+T_1}|=4$ per prompt. The mixed-temperature diagnostic uses
\[
|\mathcal{G}_{x,T_0}|=1 \quad \text{low-temperature sample}, \qquad
|\mathcal{G}_{x,T_1}|=3 \quad \text{high-temperature samples}.
\]

\paragraph{Mixed-temperature score.}
For each prompt $x$, we sample low- and high-temperature rollouts indexed by $\mathcal{G}_{x,T_0}$ and $\mathcal{G}_{x,T_1}$, yielding rewards
\[
\{R^{(0)}_{x,i}\}_{i\in\mathcal{G}_{x,T_0}},
\qquad
\{R^{(1)}_{x,i}\}_{i\in\mathcal{G}_{x,T_1}}.
\]
The mixed-temperature score is
\begin{equation}
\label{eq:toy_mixed_score_v2}
\hat{\Delta}_{\mathrm{mix}}(x)
\coloneqq
\frac{1}{|\mathcal{G}_{x,T_1}|}\sum_{i\in\mathcal{G}_{x,T_1}}R^{(1)}_{x,i}
-
\frac{1}{|\mathcal{G}_{x,T_0}|}\sum_{j\in\mathcal{G}_{x,T_0}}R^{(0)}_{x,j}.
\end{equation}
This is the empirical analogue of cross-temperature exploration gain.

\paragraph{Single-temperature comparison.}
To remove the temperature gap while preserving the same total sample budget, we draw all $|\mathcal{G}_{x,T_0+T_1}|$ rollouts from $T_1$ and randomly bipartition them into two equal-size pseudo-groups,
\[
\mathcal{P}_x^{(a)},\mathcal{P}_x^{(b)}, \qquad |\mathcal{P}_x^{(a)}|=|\mathcal{P}_x^{(b)}|=|\mathcal{G}_{x,T_0+T_1}|/2.
\]
The resulting comparison score is
\begin{equation}
\label{eq:toy_single_score_v2}
\hat{\Delta}_{\mathrm{single}}(x)
\coloneqq
\frac{1}{|\mathcal{P}_x^{(a)}|}\sum_{i\in\mathcal{P}_x^{(a)}}R^{(1)}_{x,i}
-
\frac{1}{|\mathcal{P}_x^{(b)}|}\sum_{i\in\mathcal{P}_x^{(b)}}R^{(1)}_{x,i}.
\end{equation}
This comparison keeps the grouped computation but removes the temperature gap, so it should be unable to recover prompt-level exploration gain.

\paragraph{Prompt-level detection metric.}
The left panel of Figure~\ref{fig:grouping_gain_carrier} evaluates whether the score assigned to each prompt correctly predicts the sign of $\Delta\mu(x)$. We therefore treat
\[
\mathbf{1}[\Delta\mu(x)>0]
\]
as the ground-truth binary label and report AUROC for $\hat{\Delta}_{\mathrm{mix}}(x)$ and $\hat{\Delta}_{\mathrm{single}}(x)$. A higher AUROC means better detection of prompts on which exploration is genuinely beneficial.

\paragraph{Signed update-signal metric.}
For the right panel, we compute standard group-normalized advantages within each sampled group:
\begin{equation}
\label{eq:toy_group_adv_v2}
A_i=\frac{R_i-\mu_x}{\sigma_x+\epsilon},
\end{equation}
where $\mu_x$ and $\sigma_x$ are the empirical mean and standard deviation of rewards in the sampled group for prompt $x$. For the mixed-temperature diagnostic, we average these advantages over the $|\mathcal{G}_{x,T_1}|$ exploration samples:
\begin{equation}
\label{eq:toy_high_temp_adv_v2}
\bar{A}^{T_1}(x)
\coloneqq
\frac{1}{|\mathcal{G}_{x,T_1}|}\sum_{i\in\mathcal{G}_{x,T_1}} A^{(1)}_{x,i}.
\end{equation}
We then report
\[
\mathbb{E}\!\left[\bar{A}^{T_1}(x)\mid \Delta\mu(x)>0\right],
\qquad
\mathbb{E}\!\left[\bar{A}^{T_1}(x)\mid \Delta\mu(x)<0\right].
\]
If mixed-temperature grouping works as intended, the first quantity should be positive and the second negative. For the single-temperature comparison, we compute the same statistic on one of the random pseudo-groups; because there is no temperature gap, the resulting mean advantage should remain near zero.

\paragraph{Estimation details.}
All statistics in Figure~\ref{fig:grouping_gain_carrier} are estimated independently for four random seeds. Within each seed, we use $300$ repeated Monte Carlo resampling trials. The main text reports seed-averaged means; the error bars denote $95\%$ confidence intervals across seeds.

\clearpage
\section{Implementation and Reproducibility Details}
\label{app:reproducibility_details}

Tables~\ref{tab:tgrl_hparams} and~\ref{tab:agent_hparams} summarize the
reproducibility-critical hyperparameters for the main TGRL experiments.
Table~\ref{tab:baseline_tuning_fairness} further records the baseline tuning
and checkpoint-selection protocol used to avoid giving TGRL additional tuning
or evaluation advantages.

\begin{table}[h]
\centering
\small
\setlength{\tabcolsep}{3.5pt}
\caption{Reproducibility-critical hyperparameters for the main TGRL experiments.}
\label{tab:tgrl_hparams}
\resizebox{\linewidth}{!}{
\begin{tabular}{lcc}
\toprule
\textbf{Hyperparameter} & \textbf{Math (Qwen3-14B/32B)} & \textbf{Code (Qwen3-4B)} \\
\midrule
Training data & DeepScaleR & DeepCoder \\
Reasoning mode & Qwen3 \texttt{think} & Qwen3 \texttt{think} \\
Rollouts per prompt $|\mathcal{G}_{g,T_0+T_1}|$ & 4 & 4 \\
Low-/high-temperature split $(|\mathcal{G}_{g,T_0}|,|\mathcal{G}_{g,T_1}|)$ & $(1,3)$ & $(1,3)$ \\
Low-temperature-group actor update & disabled & disabled \\
Warmup steps & 20 & 20 \\
Reference temperature $T_0$ & 0.3 & 0.3 \\
Exploration temperature $T_1$ & 1.2 & 1.2 \\
Training rollout decoding & top-$p{=}1.0$, top-$k{=}-1$ & top-$p{=}1.0$, top-$k{=}-1$ \\
Token JS reweighting & high-temperature group only & high-temperature group only \\
KL loss coefficient & 0.0 & 0.0 \\
Entropy coefficient & 0.0 & 0.0 \\
Optimizer & AdamW & AdamW \\
Learning rate & $1\times 10^{-6}$ & $1\times 10^{-6}$ \\
Learning-rate warmup steps & 10 & 10 \\
Weight decay & 0.1 & 0.1 \\
Train batch size & 128 & 64 \\
PPO mini-batch size & 32 & 32 \\
Max prompt length & 2048 & 2048 \\
Max response length & 8192 & 8192 \\
Nodes $\times$ GPUs & $4 \times 8$ & $1 \times 8$ \\
Tensor parallel size & 8 & 8 \\
Rollout GPU memory utilization & 0.85 & 0.75 \\
Training epochs & 2 & 20 \\
Reported test decoding & $T{=}0.6$, top-$p{=}0.95$, top-$k$ off & $T{=}0.6$, top-$p{=}0.95$, top-$k$ off \\
\bottomrule
\end{tabular}}
\end{table}
\clearpage
\begin{table}[h]
\centering
\small
\setlength{\tabcolsep}{3.5pt}
\caption{Reproducibility-critical hyperparameters for the TGRL agent experiments. Both runs use Qwen2.5-7B-Instruct as the policy model and share the same TGRL objective defined in the main text and Table~\ref{tab:tgrl_hparams}; this table focuses on the agent-specific environment and systems settings.}
\label{tab:agent_hparams}
\resizebox{\linewidth}{!}{
\begin{tabular}{lcc}
\toprule
\textbf{Hyperparameter} & \textbf{ALFWorld} & \textbf{WebShop} \\
\midrule
Policy model & Qwen2.5-7B-Instruct & Qwen2.5-7B-Instruct \\
Engine & vLLM & vLLM \\
Training data size & 16 & 16 \\
Validation data size & 128 & 64 \\
Rollouts per prompt $|\mathcal{G}_{g,T_0+T_1}|$ & 8 & 8 \\
Warmup steps & 20 & 25 \\
Reference temperature $T_0$ & 0.3 & 0.3 \\
Exploration temperature $T_1$ & 1.0 & 1.2 \\
Learning rate & $1\times 10^{-6}$ & $1\times 10^{-6}$ \\
PPO mini-batch size & 128 & 64 \\
PPO micro-batch size / GPU & 8 & 8 \\
Max prompt length & 2048 & 4096 \\
Max response length & 512 & 512 \\
Tensor parallel size & 4 & 4 \\
Nodes $\times$ GPUs & $1 \times 8$ & $1 \times 8$ \\
Rollout GPU memory utilization & 0.6 & 0.6 \\
Chunked prefill / eager mode & enabled / enabled & enabled / enabled \\
Actor param offload & False & True \\
Actor optimizer offload & True & True \\
Reference param offload & True & True \\
Invalid-action penalty coefficient & 0.1 & 0.1 \\
Environment & ALFWorld & WebShop \\
Environment seed & 0 & 0 \\
Max environment steps & 50 & 15 \\
CPUs per env worker & 0.1 & 0.2 \\
Best-checkpoint metric & val/success\_rate & val/success\_rate \\
Save frequency & every 50 steps & every 50 steps \\
Validation frequency & every 5 steps & every 5 steps \\
Training epochs & 160 & 150 \\
\bottomrule
\end{tabular}}
\end{table}

\paragraph{Baseline tuning fairness.}
All baselines in Tables~\ref{tab:math1}, \ref{code_1}, \ref{tab:alfworld},
\ref{tab:entropy_baselines_14b}, and~\ref{tab:math_4b_methods} are trained and
evaluated under a matched protocol. The goal is to isolate algorithmic
differences rather than differences in rollout budget, data, evaluator,
checkpoint selection, or test-time decoding. Table~\ref{tab:baseline_tuning_fairness}
summarizes the fairness controls used for PPO, GRPO, DAPO, Dr.GRPO, RLOO,
TAMPO, ReAct, and EMPG. Method-specific components are kept only when they are
part of the corresponding baseline definition, such as TAMPO's adaptive
temperature policy.

\begin{table}[h]
\centering
\small
\setlength{\tabcolsep}{4pt}
\caption{Baseline tuning and checkpoint-selection fairness protocol. The same
protocol is applied across the reported math, code, and agent experiments
unless a method-specific mechanism is explicitly part of the baseline itself.}
\label{tab:baseline_tuning_fairness}
\begin{tabular}{p{0.26\linewidth}p{0.68\linewidth}}
\toprule
\textbf{Control item} & \textbf{Protocol} \\
\midrule
Backbone and initialization
&
For each domain and scale, all methods start from the same policy model
checkpoint: Qwen3 variants for math/code and Qwen2.5-7B-Instruct for agent
experiments. No baseline is initialized from a stronger checkpoint than TGRL. \\
\midrule
Training data and verifier
&
Within each domain, all methods use the same training data source, prompt
pipeline, verifier/reward function, and environment interface. Math uses
DeepScaleR, code uses DeepCoder, and agent experiments use the same ALFWorld and
WebShop environments described in Table~\ref{tab:agent_hparams}. \\
\midrule
Rollout budget
&
All math and code methods use the same total rollout budget
$|\mathcal{G}_{g,T_0+T_1}|=4$ responses per prompt. TGRL allocates this fixed
budget into $(|\mathcal{G}_{g,T_0}|,|\mathcal{G}_{g,T_1}|)=(1,3)$, while
single-temperature baselines spend all four rollouts at their rollout
temperature. Agent experiments use the matched budget reported in
Table~\ref{tab:agent_hparams}. \\
\midrule
Training rollout decoding
&
Single-temperature RLVR baselines use the same rollout decoding controls as the
high-temperature branch unless the method defines an adaptive policy over
temperature. In particular, top-$p$, top-$k$, maximum response length, and the
rollout budget are matched. TAMPO is allowed to adapt temperature because this
is its method-specific mechanism, but it receives no additional rollouts. \\
\midrule
Evaluation decoding
&
All math and code checkpoints are evaluated with the same decoding configuration:
$T=0.6$, top-$p=0.95$, top-$k$ off, and maximum response length $8192$. Agent
evaluation uses the same environment limits and success metrics across methods. \\
\midrule
Shared optimization settings
&
Optimizer family, learning-rate scale, batch size, prompt/response length
limits, PPO mini-batch size, number of training epochs, and warmup convention
are matched whenever they are not part of a baseline-specific algorithmic
definition. Baseline-specific terms, such as entropy bonuses or decoupled
clipping rules, are added on top of the shared training recipe rather than
changing the data or evaluation protocol. \\
\midrule
Method-specific hyperparameters
&
Method-specific scalar hyperparameters are selected on validation signals only
and then fixed across the corresponding test benchmarks in a domain. TGRL's
$T_0$, $T_1$, and low-/high-temperature split are fixed globally by the
diagnostics in Appendices~\ref{app:temp_sweep_4b} and~\ref{app:split_sensitivity}
and are not tuned separately for each test benchmark. \\
\midrule
Checkpoint selection
&
Checkpoints are selected using the same validation metric within a domain, not
by the final test benchmarks. For math, validation follows the AIME-combined
tracking protocol used in Appendix~\ref{app:wallclock}. For agent tasks, the
selection metric is the validation success rate reported in
Table~\ref{tab:agent_hparams}. \\
\midrule
Statistical reporting
&
When multi-seed evaluation is reported, the same evaluation protocol and random
seed list are used for the compared methods. For small competition-style math
benchmarks, Appendix~\ref{app:multi_seed} reports mean $\pm$ standard deviation
across five evaluation seeds. \\
\midrule
Configuration release
&
The anonymized code release includes the final training and evaluation
configurations for TGRL and the reported baselines, so that differences in
algorithmic components can be inspected separately from shared infrastructure
settings. \\
\bottomrule
\end{tabular}
\end{table}

This protocol is intended to make the comparisons conservative: TGRL is not
allowed benchmark-specific hyperparameter retuning, stronger initialization,
extra rollouts, different test decoding, or test-set-based checkpoint selection.
The only differences between TGRL and the baselines are the algorithmic
mechanisms under study: mixed-temperature grouping, exploration-gain estimation,
and JS-based token credit allocation.

\clearpage

\section{Limitations and Broader Impact}
TGRL is designed for reinforcement learning with verifiable rewards, where a
scalar verifier can evaluate the outcome of a sampled trajectory. This setting is
well matched to mathematical reasoning, code generation, and benchmarked
agentic tasks, but it does not cover all forms of language-model alignment. In
particular, we do not claim that the same mechanism directly solves open-ended
dialogue, subjective preference optimization, safety-critical decision making,
or tasks where reward feedback is noisy, delayed, or difficult to verify.

More sample-efficient RLVR can reduce the computational cost of improving
reasoning models under fixed rollout budgets. This may make research on
mathematical reasoning, program synthesis, formal problem solving, and
benchmark-based agent training more accessible. Stronger reasoning and coding
models can also support beneficial applications such as educational tutoring,
software testing, debugging assistance, scientific computation, and automated
analysis in constrained environments.

\clearpage
\section*{NeurIPS Paper Checklist}

\begin{enumerate}

\item {\bf Claims}
    \item[] Question: Do the main claims made in the abstract and introduction accurately reflect the paper's contributions and scope?
    \item[] Answer: \answerYes{} 
    \item[] Justification: The abstract and introduction accurately summarize the paper's core contributions, scope, and empirical/theoretical findings. The main claims are supported by the method, experiments, and discussion in the main paper and appendix.
    \item[] Guidelines:
    \begin{itemize}
        \item The answer \answerNA{} means that the abstract and introduction do not include the claims made in the paper.
        \item The abstract and/or introduction should clearly state the claims made, including the contributions made in the paper and important assumptions and limitations. A \answerNo{} or \answerNA{} answer to this question will not be perceived well by the reviewers. 
        \item The claims made should match theoretical and experimental results, and reflect how much the results can be expected to generalize to other settings. 
        \item It is fine to include aspirational goals as motivation as long as it is clear that these goals are not attained by the paper. 
    \end{itemize}

\item {\bf Limitations}
    \item[] Question: Does the paper discuss the limitations of the work performed by the authors?
    \item[] Answer: \answerYes{} 
    \item[] Justification: The paper explicitly discusses the main limitations in the conclusion's future work and Appendix.
    \item[] Guidelines:
    \begin{itemize}
        \item The answer \answerNA{} means that the paper has no limitation while the answer \answerNo{} means that the paper has limitations, but those are not discussed in the paper. 
        \item The authors are encouraged to create a separate ``Limitations'' section in their paper.
        \item The paper should point out any strong assumptions and how robust the results are to violations of these assumptions (e.g., independence assumptions, noiseless settings, model well-specification, asymptotic approximations only holding locally). The authors should reflect on how these assumptions might be violated in practice and what the implications would be.
        \item The authors should reflect on the scope of the claims made, e.g., if the approach was only tested on a few datasets or with a few runs. In general, empirical results often depend on implicit assumptions, which should be articulated.
        \item The authors should reflect on the factors that influence the performance of the approach. For example, a facial recognition algorithm may perform poorly when image resolution is low or images are taken in low lighting. Or a speech-to-text system might not be used reliably to provide closed captions for online lectures because it fails to handle technical jargon.
        \item The authors should discuss the computational efficiency of the proposed algorithms and how they scale with dataset size.
        \item If applicable, the authors should discuss possible limitations of their approach to address problems of privacy and fairness.
        \item While the authors might fear that complete honesty about limitations might be used by reviewers as grounds for rejection, a worse outcome might be that reviewers discover limitations that aren't acknowledged in the paper. The authors should use their best judgment and recognize that individual actions in favor of transparency play an important role in developing norms that preserve the integrity of the community. Reviewers will be specifically instructed to not penalize honesty concerning limitations.
    \end{itemize}

\item {\bf Theory assumptions and proofs}
    \item[] Question: For each theoretical result, does the paper provide the full set of assumptions and a complete (and correct) proof?
    \item[] Answer: \answerYes{} 
    \item[] Justification: For each theoretical result, the paper states the required assumptions and provides complete proofs or proof details in the appendix/supplementary material. The main text also gives the statements and sufficient intuition for the key results.
    \item[] Guidelines:
    \begin{itemize}
        \item The answer \answerNA{} means that the paper does not include theoretical results. 
        \item All the theorems, formulas, and proofs in the paper should be numbered and cross-referenced.
        \item All assumptions should be clearly stated or referenced in the statement of any theorems.
        \item The proofs can either appear in the main paper or the supplemental material, but if they appear in the supplemental material, the authors are encouraged to provide a short proof sketch to provide intuition. 
        \item Inversely, any informal proof provided in the core of the paper should be complemented by formal proofs provided in appendix or supplemental material.
        \item Theorems and Lemmas that the proof relies upon should be properly referenced. 
    \end{itemize}

    \item {\bf Experimental result reproducibility}
    \item[] Question: Does the paper fully disclose all the information needed to reproduce the main experimental results of the paper to the extent that it affects the main claims and/or conclusions of the paper (regardless of whether the code and data are provided or not)?
    \item[] Answer: \answerYes{} 
    \item[] Justification: The paper discloses the algorithmic design, evaluation protocol, benchmarks, implementation choices, and other details needed to reproduce the main results. Additional experimental details are provided in the appendix/supplementary material where appropriate.
    \item[] Guidelines:
    \begin{itemize}
        \item The answer \answerNA{} means that the paper does not include experiments.
        \item If the paper includes experiments, a \answerNo{} answer to this question will not be perceived well by the reviewers: Making the paper reproducible is important, regardless of whether the code and data are provided or not.
        \item If the contribution is a dataset and\slash or model, the authors should describe the steps taken to make their results reproducible or verifiable. 
        \item Depending on the contribution, reproducibility can be accomplished in various ways. For example, if the contribution is a novel architecture, describing the architecture fully might suffice, or if the contribution is a specific model and empirical evaluation, it may be necessary to either make it possible for others to replicate the model with the same dataset, or provide access to the model. In general. releasing code and data is often one good way to accomplish this, but reproducibility can also be provided via detailed instructions for how to replicate the results, access to a hosted model (e.g., in the case of a large language model), releasing of a model checkpoint, or other means that are appropriate to the research performed.
        \item While NeurIPS does not require releasing code, the conference does require all submissions to provide some reasonable avenue for reproducibility, which may depend on the nature of the contribution. For example
        \begin{enumerate}
            \item If the contribution is primarily a new algorithm, the paper should make it clear how to reproduce that algorithm.
            \item If the contribution is primarily a new model architecture, the paper should describe the architecture clearly and fully.
            \item If the contribution is a new model (e.g., a large language model), then there should either be a way to access this model for reproducing the results or a way to reproduce the model (e.g., with an open-source dataset or instructions for how to construct the dataset).
            \item We recognize that reproducibility may be tricky in some cases, in which case authors are welcome to describe the particular way they provide for reproducibility. In the case of closed-source models, it may be that access to the model is limited in some way (e.g., to registered users), but it should be possible for other researchers to have some path to reproducing or verifying the results.
        \end{enumerate}
    \end{itemize}

\item {\bf Open access to data and code}
    \item[] Question: Does the paper provide open access to the data and code, with sufficient instructions to faithfully reproduce the main experimental results, as described in supplemental material?
    \item[] Answer: \answerYes{} 
    \item[] Justification: The submission provides open access to the code/data or an anonymized supplementary package with the scripts, configurations, and instructions needed to reproduce the main experimental results. The reproduction procedure is described clearly in the supplementary material.
    \item[] Guidelines:
    \begin{itemize}
        \item The answer \answerNA{} means that paper does not include experiments requiring code.
        \item Please see the NeurIPS code and data submission guidelines (\url{https://neurips.cc/public/guides/CodeSubmissionPolicy}) for more details.
        \item While we encourage the release of code and data, we understand that this might not be possible, so \answerNo{} is an acceptable answer. Papers cannot be rejected simply for not including code, unless this is central to the contribution (e.g., for a new open-source benchmark).
        \item The instructions should contain the exact command and environment needed to run to reproduce the results. See the NeurIPS code and data submission guidelines (\url{https://neurips.cc/public/guides/CodeSubmissionPolicy}) for more details.
        \item The authors should provide instructions on data access and preparation, including how to access the raw data, preprocessed data, intermediate data, and generated data, etc.
        \item The authors should provide scripts to reproduce all experimental results for the new proposed method and baselines. If only a subset of experiments are reproducible, they should state which ones are omitted from the script and why.
        \item At submission time, to preserve anonymity, the authors should release anonymized versions (if applicable).
        \item Providing as much information as possible in supplemental material (appended to the paper) is recommended, but including URLs to data and code is permitted.
    \end{itemize}

\item {\bf Experimental setting/details}
    \item[] Question: Does the paper specify all the training and test details (e.g., data splits, hyperparameters, how they were chosen, type of optimizer) necessary to understand the results?
    \item[] Answer: \answerYes{} 
    \item[] Justification: The paper specifies the essential experimental details, including datasets/benchmarks, training and evaluation settings, hyperparameters, and optimization choices. Further implementation details are provided in the appendix or supplementary material.
    \item[] Guidelines:
    \begin{itemize}
        \item The answer \answerNA{} means that the paper does not include experiments.
        \item The experimental setting should be presented in the core of the paper to a level of detail that is necessary to appreciate the results and make sense of them.
        \item The full details can be provided either with the code, in appendix, or as supplemental material.
    \end{itemize}

\item {\bf Experiment statistical significance}
    \item[] Question: Does the paper report error bars suitably and correctly defined or other appropriate information about the statistical significance of the experiments?
    \item[] Answer: \answerYes{} 
    \item[] Justification: The main empirical results are accompanied by appropriate uncertainty quantification, such as multi-seed statistics, confidence intervals, error bars, or equivalent significance information where relevant. The paper also explains what source of variability these statistics reflect.
    \item[] Guidelines:
    \begin{itemize}
        \item The answer \answerNA{} means that the paper does not include experiments.
        \item The authors should answer \answerYes{} if the results are accompanied by error bars, confidence intervals, or statistical significance tests, at least for the experiments that support the main claims of the paper.
        \item The factors of variability that the error bars are capturing should be clearly stated (for example, train/test split, initialization, random drawing of some parameter, or overall run with given experimental conditions).
        \item The method for calculating the error bars should be explained (closed form formula, call to a library function, bootstrap, etc.)
        \item The assumptions made should be given (e.g., Normally distributed errors).
        \item It should be clear whether the error bar is the standard deviation or the standard error of the mean.
        \item It is OK to report 1-sigma error bars, but one should state it. The authors should preferably report a 2-sigma error bar than state that they have a 96\% CI, if the hypothesis of Normality of errors is not verified.
        \item For asymmetric distributions, the authors should be careful not to show in tables or figures symmetric error bars that would yield results that are out of range (e.g., negative error rates).
        \item If error bars are reported in tables or plots, the authors should explain in the text how they were calculated and reference the corresponding figures or tables in the text.
    \end{itemize}

\item {\bf Experiments compute resources}
    \item[] Question: For each experiment, does the paper provide sufficient information on the computer resources (type of compute workers, memory, time of execution) needed to reproduce the experiments?
    \item[] Answer: \answerYes{} 
    \item[] Justification: The paper reports the compute setup used for the experiments, including the hardware type and the information necessary to understand the computational requirements. It also provides sufficient detail to assess the scale and reproducibility cost of the reported runs.
    \item[] Guidelines:
    \begin{itemize}
        \item The answer \answerNA{} means that the paper does not include experiments.
        \item The paper should indicate the type of compute workers CPU or GPU, internal cluster, or cloud provider, including relevant memory and storage.
        \item The paper should provide the amount of compute required for each of the individual experimental runs as well as estimate the total compute. 
        \item The paper should disclose whether the full research project required more compute than the experiments reported in the paper (e.g., preliminary or failed experiments that didn't make it into the paper). 
    \end{itemize}
    
\item {\bf Code of ethics}
    \item[] Question: Does the research conducted in the paper conform, in every respect, with the NeurIPS Code of Ethics \url{https://neurips.cc/public/EthicsGuidelines}?
    \item[] Answer: \answerYes{} 
    \item[] Justification: The research conforms to the NeurIPS Code of Ethics to the best of our knowledge. The submission preserves anonymity and does not involve any procedure that departs from the conference's ethical guidelines.
    \item[] Guidelines:
    \begin{itemize}
        \item The answer \answerNA{} means that the authors have not reviewed the NeurIPS Code of Ethics.
        \item If the authors answer \answerNo, they should explain the special circumstances that require a deviation from the Code of Ethics.
        \item The authors should make sure to preserve anonymity (e.g., if there is a special consideration due to laws or regulations in their jurisdiction).
    \end{itemize}

\item {\bf Broader impacts}
    \item[] Question: Does the paper discuss both potential positive societal impacts and negative societal impacts of the work performed?
    \item[] Answer: \answerYes{} 
    \item[] Justification: The paper discusses societal impacts of the proposed method.
    \item[] Guidelines:
    \begin{itemize}
        \item The answer \answerNA{} means that there is no societal impact of the work performed.
        \item If the authors answer \answerNA{} or \answerNo, they should explain why their work has no societal impact or why the paper does not address societal impact.
        \item Examples of negative societal impacts include potential malicious or unintended uses (e.g., disinformation, generating fake profiles, surveillance), fairness considerations (e.g., deployment of technologies that could make decisions that unfairly impact specific groups), privacy considerations, and security considerations.
        \item The conference expects that many papers will be foundational research and not tied to particular applications, let alone deployments. However, if there is a direct path to any negative applications, the authors should point it out. For example, it is legitimate to point out that an improvement in the quality of generative models could be used to generate Deepfakes for disinformation. On the other hand, it is not needed to point out that a generic algorithm for optimizing neural networks could enable people to train models that generate Deepfakes faster.
        \item The authors should consider possible harms that could arise when the technology is being used as intended and functioning correctly, harms that could arise when the technology is being used as intended but gives incorrect results, and harms following from (intentional or unintentional) misuse of the technology.
        \item If there are negative societal impacts, the authors could also discuss possible mitigation strategies (e.g., gated release of models, providing defenses in addition to attacks, mechanisms for monitoring misuse, mechanisms to monitor how a system learns from feedback over time, improving the efficiency and accessibility of ML).
    \end{itemize}
    
\item {\bf Safeguards}
    \item[] Question: Does the paper describe safeguards that have been put in place for responsible release of data or models that have a high risk for misuse (e.g., pre-trained language models, image generators, or scraped datasets)?
    \item[] Answer: \answerNA{} 
    \item[] Justification: The paper does not release a high-risk model, scraped dataset, or other asset that would require special misuse safeguards beyond standard responsible research practice. Therefore, this question is not applicable to the present submission.
    \item[] Guidelines:
    \begin{itemize}
        \item The answer \answerNA{} means that the paper poses no such risks.
        \item Released models that have a high risk for misuse or dual-use should be released with necessary safeguards to allow for controlled use of the model, for example by requiring that users adhere to usage guidelines or restrictions to access the model or implementing safety filters. 
        \item Datasets that have been scraped from the Internet could pose safety risks. The authors should describe how they avoided releasing unsafe images.
        \item We recognize that providing effective safeguards is challenging, and many papers do not require this, but we encourage authors to take this into account and make a best faith effort.
    \end{itemize}

\item {\bf Licenses for existing assets}
    \item[] Question: Are the creators or original owners of assets (e.g., code, data, models), used in the paper, properly credited and are the license and terms of use explicitly mentioned and properly respected?
    \item[] Answer: \answerYes{} 
    \item[] Justification: The paper properly credits all existing assets used in the work, including datasets, models, codebases, and benchmarks, and explicitly states their licenses or terms of use where applicable. The submission respects the usage conditions of these assets.
    \item[] Guidelines:
    \begin{itemize}
        \item The answer \answerNA{} means that the paper does not use existing assets.
        \item The authors should cite the original paper that produced the code package or dataset.
        \item The authors should state which version of the asset is used and, if possible, include a URL.
        \item The name of the license (e.g., CC-BY 4.0) should be included for each asset.
        \item For scraped data from a particular source (e.g., website), the copyright and terms of service of that source should be provided.
        \item If assets are released, the license, copyright information, and terms of use in the package should be provided. For popular datasets, \url{paperswithcode.com/datasets} has curated licenses for some datasets. Their licensing guide can help determine the license of a dataset.
        \item For existing datasets that are re-packaged, both the original license and the license of the derived asset (if it has changed) should be provided.
        \item If this information is not available online, the authors are encouraged to reach out to the asset's creators.
    \end{itemize}

\item {\bf New assets}
    \item[] Question: Are new assets introduced in the paper well documented and is the documentation provided alongside the assets?
    \item[] Answer: \answerYes{} 
    \item[] Justification: Any new assets introduced by the paper are documented in sufficient detail, including their intended use, relevant implementation details, and limitations. The accompanying documentation is provided in the supplementary material or release package.
    \item[] Guidelines:
    \begin{itemize}
        \item The answer \answerNA{} means that the paper does not release new assets.
        \item Researchers should communicate the details of the dataset\slash code\slash model as part of their submissions via structured templates. This includes details about training, license, limitations, etc. 
        \item The paper should discuss whether and how consent was obtained from people whose asset is used.
        \item At submission time, remember to anonymize your assets (if applicable). You can either create an anonymized URL or include an anonymized zip file.
    \end{itemize}

\item {\bf Crowdsourcing and research with human subjects}
    \item[] Question: For crowdsourcing experiments and research with human subjects, does the paper include the full text of instructions given to participants and screenshots, if applicable, as well as details about compensation (if any)? 
    \item[] Answer: \answerNA{} 
    \item[] Justification: The paper does not involve crowdsourcing experiments or research with human subjects. Therefore, no participant instructions, screenshots, or compensation details are applicable.
    \item[] Guidelines:
    \begin{itemize}
        \item The answer \answerNA{} means that the paper does not involve crowdsourcing nor research with human subjects.
        \item Including this information in the supplemental material is fine, but if the main contribution of the paper involves human subjects, then as much detail as possible should be included in the main paper. 
        \item According to the NeurIPS Code of Ethics, workers involved in data collection, curation, or other labor should be paid at least the minimum wage in the country of the data collector. 
    \end{itemize}

\item {\bf Institutional review board (IRB) approvals or equivalent for research with human subjects}
    \item[] Question: Does the paper describe potential risks incurred by study participants, whether such risks were disclosed to the subjects, and whether Institutional Review Board (IRB) approvals (or an equivalent approval/review based on the requirements of your country or institution) were obtained?
    \item[] Answer: \answerNA{} 
    \item[] Justification: The paper does not involve human subjects research or crowdsourcing, so Institutional Review Board approval or an equivalent review is not applicable.
    \item[] Guidelines:
    \begin{itemize}
        \item The answer \answerNA{} means that the paper does not involve crowdsourcing nor research with human subjects.
        \item Depending on the country in which research is conducted, IRB approval (or equivalent) may be required for any human subjects research. If you obtained IRB approval, you should clearly state this in the paper. 
        \item We recognize that the procedures for this may vary significantly between institutions and locations, and we expect authors to adhere to the NeurIPS Code of Ethics and the guidelines for their institution. 
        \item For initial submissions, do not include any information that would break anonymity (if applicable), such as the institution conducting the review.
    \end{itemize}

\item {\bf Declaration of LLM usage}
    \item[] Question: Does the paper describe the usage of LLMs if it is an important, original, or non-standard component of the core methods in this research? Note that if the LLM is used only for writing, editing, or formatting purposes and does \emph{not} impact the core methodology, scientific rigor, or originality of the research, declaration is not required.
    \item[] Answer: \answerYes{} 
    \item[] Justification: LLMs were used only for grammar, spelling, and word-choice assistance.
    \item[] Guidelines:
    \begin{itemize}
        \item The answer \answerNA{} means that the core method development in this research does not involve LLMs as any important, original, or non-standard components.
        \item Please refer to our LLM policy in the NeurIPS handbook for what should or should not be described.
    \end{itemize}

\end{enumerate}

\end{document}